\documentclass{article}

\usepackage{microtype}
\usepackage{graphicx}
\usepackage{booktabs} % for professional tables

\usepackage{hyperref}
\usepackage{amsfonts}       % blackboard math symbols
\usepackage{url}
\usepackage{amsmath}
\usepackage{amsthm}
\usepackage{wrapfig}
\usepackage{mathtools}
\usepackage{tikz}
\usepackage{subfig}
\usepackage{algorithmic}
\usepackage{subfig}
\usepackage{footmisc}
\newcounter{thm_counter}
\newcounter{lem_counter}
\newcounter{pro_counter}

\newcounter{ass_counter}
\numberwithin{ass_counter}{section}

\newcounter{exam_counter}
\newtheorem{theorem}[thm_counter]{Theorem}%[section]
\newtheorem{proposition}[pro_counter]{Proposition}%[section]
\newtheorem{lemma}[lem_counter]{Lemma}%[Lemma]

\newtheorem{assumption}[ass_counter]{Assumption}

\newtheorem{example}[exam_counter]{Example}

\newcommand{\ns}{{|\mathcal{S}|}}
\newcommand{\nsa}{{|\mathcal{S}||\mathcal{A}|}}

\newcommand{\tb}[1]{{\textbf{#1}}}

\usepackage{mathtools}
\usepackage{physics}
\usepackage{mathrsfs}
\mathtoolsset{showonlyrefs}
\newcommand{\E}{\mathbb{E}}
\newcommand{\C}{\mathbb{C}}
\newcommand{\R}{\mathbb{R}}

\usepackage[inline,ignoremode]{trackchanges}
\addeditor{shangtong}

\usepackage[accepted]{icml2021}

\icmltitlerunning{Average-Reward Off-Policy Policy Evaluation with Function Approximation}

\begin{document}
\twocolumn[
\icmltitle{Average-Reward Off-Policy Policy Evaluation with Function Approximation
}



% in order of appearance, and this is the preferred way.
\icmlsetsymbol{equal}{*}

\begin{icmlauthorlist}
\icmlauthor{Shangtong Zhang}{ox,equal}
\icmlauthor{Yi Wan}{ua,equal}
\icmlauthor{Richard S. Sutton}{ua}
\icmlauthor{Shimon Whiteson}{ox}
\end{icmlauthorlist}

\icmlaffiliation{ua}{University of Alberta}
\icmlaffiliation{ox}{University of Oxford}

\icmlcorrespondingauthor{Shangtong Zhang}{{{\\shangtong.zhang}@cs.ox.ac.uk}}
\icmlcorrespondingauthor{Yi Wan}{{wan6@ualberta.ca}}

\icmlkeywords{}

\vskip 0.3in
]



\printAffiliationsAndNotice{\icmlEqualContribution} % otherwise use the standard text.

\begin{abstract}
We consider off-policy policy evaluation with function approximation (FA) in average-reward MDPs, 
where the goal is to estimate both the reward rate and the differential value function. 
For this problem, bootstrapping is necessary and, along with off-policy learning and FA, results in the deadly triad \citep{sutton2018reinforcement}.
To address the deadly triad, we propose
two novel algorithms,
reproducing the celebrated success of Gradient TD algorithms in the average-reward setting.
In terms of estimating the differential value function, the algorithms are the first convergent off-policy linear function approximation algorithms.
In terms of estimating the reward rate,
the algorithms are the first convergent off-policy linear function approximation algorithms that do not require estimating the density ratio.
We demonstrate empirically the advantage of the proposed algorithms, 
as well as their nonlinear variants,
over a competitive density-ratio-based approach,
in a simple domain as well as challenging robot simulation tasks.

\end{abstract}

\section{Introduction}
% The average-reward formulation of  (MDPs) (see, e.g., \citet{puterman2014markov}) provides 
% has wide real-world applications \yw{cite more}
% \citep{murphy2016batch}.
% \citep{mahadevan1996average}.  
A fundamental problem in average-reward Markov Decision Processes (MDPs, see, e.g., \citet{puterman2014markov}) is \emph{policy evaluation},
that is,
estimating, for a given policy, the \emph{reward rate} and the \emph{differential value function}.
The reward rate of a policy is the average reward per step and thus measures the policy's long term performance.
The differential value function summarizes 
the expected cumulative future excess rewards, 
which are the differences between received rewards and the reward rate.
The solution of the policy evaluation problem is  interesting in itself because it provides a useful performance metric, the reward rate, for a given policy.
% but also is the corner stone of the more ambitious \emph{control} problem, 
% where we aim to generate a policy maximizing the reward rate,
In addition, it is an essential part of many \emph{control} algorithms, which aim to generate a policy that maximizes the reward rate
by iteratively improving the policy using its estimated differential value function (see, e.g., \citet{howard1960dynamic,konda2002thesis,abbasi2019politex}).

% Two fundamental problems in average-reward MDPs are 1) (\emph{policy evaluation}) estimating, for a given policy, reward rate and the differential value function  , and 2)  (\emph{control}) generating a policy that maximizes the reward rate. The reward rate of a policy is the average-reward per step, and thus measures the policy's long term performance. The differential value function for a given policy maps from a state, or state-action pair, to the expected cumulative future excess rewards, which are the the differences between received rewards and the reward rate. Differential value function is a crutial part of many control algorithms, including actor-critic algorithms \yw{cite konda tsitsiklis} policy iteration algorithms \yw{cite howard} and simulation-based policy iteration algorithms. \yw{cite politex}. The
% policy evaluation problem is not only interesting of itself, but also serves as a corner stone for the more ambitious control problem. 
% The average-reward is a useful metric for measuring the utility of a given policy,
% and the differential value function can be used for policy improvement.
% Algorithms solving these two problems can be divided into three classes, learning, planning or combined methods. For learning algorithms, the policy being executed.
One typical approach in policy evaluation is to learn from real experience directly,
without knowing or learning a model.
If the policy followed to generate experience (behavior policy) is the same as the policy of interest (target policy), then this approach yields an \emph{on-policy} method; 
otherwise, it is \emph{off-policy}.
% in which case the experience is generated 
% by one or multiple,
% by a possibly nonstationary and even unknown behavior policy.
% Off-policy methods are usually more challenging but are also more desired to be solved than the on-policy ones if
% more sample-efficient than on-policy ones \citep{lin1992self,sutton2011horde} and are 
Off-policy methods are usually more practical in settings in which
following bad policies incurs prohibitively high cost \citep{dulac2019challenges}.
% For these two problems, typical solution methods include 1) planning with a given model, 2) learning directly from real experience, and 3) learning a model from real experience and planning with the learned model. In the latter two cases, if the policy followed to generate experience (behavior policy) is the same as the policy of interest (target policy), then the solution method is called an \emph{on-policy} method. Otherwise, it is called an \emph{off-policy} method. Off-policy methods are usually more sample efficient than on-policy ones and are more amenable to problems where following bad policies incurs prohibitively high cost \sz{cite sth like autonomous driving}.
For policy evaluation,
we can use either tabular methods,
which maintain a look-up table to store quantities of interest (e.g., the differential values for all states) separately,
or use function approximation,
which represents these quantities collectively,
possibly in a more efficient way (e.g., using a neural network).
Function approximation methods are necessary for MDPs with large state and/or action spaces because they are scalable in the size of these spaces and also generalize to states and actions that are not in the data \citep{mnih2015human,silver2016mastering}.
Finally, for the policy evaluation problem in average reward MDPs, the agent's stream of experience never terminates and thus actual returns cannot be obtained. Because of this, learning algorithms have to bootstrap, that is, the estimated values must be updated towards targets that include existing estimated values instead of actual returns.

In this paper, we consider methods for solving the average-reward policy evaluation problem with all the above three elements (off-policy learning, function approximation and bootstrapping), which comprise the deadly triad (see Chapter 11 of \citet{sutton2018reinforcement} and Section~\ref{sec:sgq}).
% In particular, given data generated by one or multiple, possibly unknown policies, and a target policy, our goal is to estimate both the reward rate and the differential value function of the target policy.
The main contributions of this paper are two newly proposed methods to break this deadly triad in the average-reward setting,
both of which are inspired by the celebrated success of the Gradient TD family of algorithms \citep{sutton2009convergent,sutton2009fast} in breaking the deadly triad in the discounted setting.

Few methods exist for learning differential value functions.
These are either on-policy linear function approximation methods \citep{tsitsiklis1999average,konda2002thesis,yu2009convergence,abbasi2019politex} or off-policy tabular methods \citep{wan2020learning}.
The on-policy methods 
use the empirical average of received rewards as an estimate for the reward rate.
Thus they are not straightforward to extend to the off-policy case.
And, as we show later with a counterexample, the naive extension of the off-policy tabular method by \citet{wan2020learning} to the linear function approximation setting can diverge,
exemplifying the deadly triad. 
By contrast,
\emph{the two algorithms we propose are the first provably convergent methods for learning the differential value function via off-policy linear function approximation}.

All existing methods for estimating reward rate in off-policy function approximation setting
% Reward rate estimation in receives increasing attention recently. In order to estimate the reward rate,
require learning the \emph{density ratio}, i.e., the ratio between the stationary distribution of the target policy and the sampling distribution \citep{liu2018breaking,zhang2020gendice,zhang2020gradientdice,mousavi2020blackbox,lazic2020maximum}. 
Interestingly, while density-ratio-based methods dominate off-policy policy evaluation with function approximation in average-reward MDPs, in the discounted MDPs,
both density-ratio-based \citep{hallak2017consistent,liu2018breaking,gelada2019off,nachum2019dualdice,uehara2019minimax,xie2019towards,tang2019doubly,zhang2020gendice,zhang2020gradientdice} and value-based \citep{baird1995residual,sutton2009convergent,sutton2009fast,sutton2016emphatic,ope:thomas2015high,jiang2015doubly} methods have succeeded. It thus remains unknown whether a convergent value-based method could be found for such a problem and if it exists, how it performs compared with density-ratio-based methods.
\emph{The two algorithms we propose are the first provably convergent differential-value-based methods for reward rate estimation via off-policy linear function approximation},
which answer the question affirmatively.
Furthermore,
our empirical study shows that our value-based methods consistently outperform a competitive density-ratio-based approach, 
GradientDICE \citep{zhang2020gradientdice}, 
in the tested domains,
including both a simple Markov chain and challenging robot simulation tasks.

\section{Background}
In this paper, we use $\norm{\cdot}_M$ to denote the vector norm induced by a positive definite matrix $M$, i.e., $\norm{x}_M = \sqrt{x^\top M x}$. 
We also use $\norm{\cdot}_M$ to denote the corresponding induced matrix norm.
When $M = I$, we ignore the subscript $I$ and write $\norm{\cdot}$ for simplicity.
All vectors are column vectors.
$\tb{0}$ denotes an all-zero vector whose dimension can be deduced from the context.
$\tb{1}$ is similarly defined.
When it does not confuse,
we use a function and a vector interchangeably.
For example, if $f$ is a function from $\mathcal{X}$ to $\R$,
we also use $f$ to denote the corresponding vector in $\R^{|\mathcal{X}|}$.
% We use $X^\dagger$ to denote the Moore-Penrose pseudoinverse of $X$,
% which always exists regardless of the rank of $X$.

We consider an infinite horizon MDP with a finite state space $\mathcal{S}$,
a finite action space $\mathcal{A}$,
a reward function $r: \mathcal{S} \times \mathcal{A} \to \R$,
and a transition kernel $p: \mathcal{S} \times \mathcal{S} \times \mathcal{A} \to [0, 1]$.
When an agent follows a policy $\pi: \mathcal{A} \times \mathcal{S} \to [0, 1]$ in the MDP, at time step $t$, 
the agent observes a state $S_t$, takes an action $A_t \sim \pi(\cdot | S_t)$,
receives a reward $r(S_t, A_t)$, proceeds to the next time step and observes the next state $S_{t+1} \sim p(\cdot | S_t, A_t)$.
The reward rate of policy $\pi$ is defined as
\begin{align} \label{eq: reward rate definition}
\textstyle{r_\pi \doteq C \text{-} \lim_{t \to \infty} \E[ r(S_t, A_t) \mid \pi, S_0]},
\end{align}
where $C \text{-}\! \lim_{T \to \infty} z_T \doteq \lim_{T \to \infty} \frac{1}{T + 1} \sum_{i=0}^T z_i$ is the Cesaro limit. The Cesaro limit in \eqref{eq: reward rate definition} is assumed to exist and is independent of $S_0$. The most general assumption that guarantees these is the following one:
\begin{assumption}
\label{assu:chain}
Policy $\pi$ induces a unichain. 
\end{assumption}
The action-value function in the average-reward setting is known as the differential action-value function and is defined as $q_\pi(s, a) \doteq \textstyle{C\text{-}\lim_{T \to \infty} } \textstyle{ \sum_{t = 0}^T \E[r(S_{t}, A_{t}) - r_\pi \mid S_0 = s, A_0 = a ]}$. 
% \begin{align}
% q_\pi(s, a) \doteq &\textstyle{C\text{-}\lim_{T \to \infty} } \\ & \textstyle{ \sum_{t = 0}^T \E[r(S_{t}, A_{t}) - r_\pi \mid S_0 = s, A_0 = a ]}.
% \end{align}
Note that if a stronger ergodic chain assumption is used instead, 
the Cesaro limit in defining $r_\pi$ and $q_\pi$ is equivalent to the normal limit.
The action-value Bellman equation is
\begin{align}
\label{eq:differential-bellman}
q = r - \bar{r}\tb{1} + P_\pi q,
\end{align}
where $q \in \R^\nsa$ and $\bar{r} \in \R$ are free variables and $P_\pi \in \R^{\nsa \times \nsa}$ is the transition matrix, that is,
$P_\pi((s, a), (s', a')) \doteq p(s'|s, a) \pi(a'|s')$.
It is well-known \citep{puterman2014markov} that $r = r_\pi$ is the unique solution for $r$ and all the solutions for $q$ form a set $\{q_\pi + c \tb{1}: \forall c \in \R \}$.

% In this paper,
% we consider the problem of Off-policy Evaluation (OPE),
% where the agent estimates the reward rate of policy $\pi$, $r_\pi$, without actually following $\pi$ in the MDP to collect data.
% Further, instead of the online learning setting, where the data is a stream of continuous experience $\{S_t, A_t, R_{t+1}\}_{t\geq 0}$, 
% we consider the \emph{behavior-agnostic off-policy learning setting} \citep{nachum2019dualdice}, 
% where a sampling distribution $d_{\mu \pi}$ is given and the agent draws samples from it. 
% In this paper, we consider the offline learning setting, 
% collection are separated processes and the agent doesn't collect new data during the learning process. 
% The agent 
In this paper, we consider a special off-policy learning setting, 
where the agent learns from i.i.d. samples drawn from a given sampling distribution.
% constructed from data collected before the learning process.
% This is different from the online learning setting, where the agent simultaneously interacts with an MDP and learns from the interactions. Specifically, we consider an instance of offline learning setting, named \emph{behavior-agnostic off-policy learning setting} \citep{nachum2019dualdice}, where the data can be collected 
% % by following one or multiple, 
% by following an possibly unstationary and even unknown behavior policy. 
In particular, at the $k$-th iteration, the agent draws a sample $(S_k, A_k, R_k, S_k', A_k')$ from a given sampling distribution $d_{\mu \pi}$. 
Distribution $d_{\mu\pi}$ can be any distribution satisfying 
\begin{assumption} \label{assu: positive dmu}
$R_k = r(S_k, A_k)$, $S_k' \sim p(\cdot | S_k, A_k)$, $A_{k}' \sim \pi(\cdot | S_k')$, and $d_\mu(s, a) > 0$ for all $(s, a)$,
\end{assumption}
where $d_\mu(s, a)$ denotes the marginal distribution of $(S_k, A_k)$. The last part of Assumption~\ref{assu: positive dmu} means that every state-action pair is possible to be sampled. 
This is a necessary condition for learning the differential value function accurately for all state-action pairs.
In the rest of the paper, the expectation
$\E$ is taken w.r.t. $d_{\mu\pi}$.

If no sampling distribution is given, one could instead draw samples in the following way.
% $(S_k, A_k, R_k, S_k', A_k')$ drawn from some sampling distribution can be obtained approximately in the following way.
% from $d_{\mu\pi}$ can be approximated in the following way. 
First randomly sample $(S_k, A_k, R_k, S_k')$ from a batch of transitions 
% $(s_i, a_i, r_i, s_i')$ 
collected by one or multiple agents, with all agents following possibly different unknown policies in the same MDP. Then sample $A_k' \sim \pi(\cdot \mid S_k')$. 
Assuming that the number of all state-action pairs in the batch grows to infinity as the batch size grows to infinity
then sampling from the batch is approximately equivalent to sampling from some distribution satisfying Assumption~\ref{assu: positive dmu}.
% , then the approximation using a finite batch can be arbitrarily accurate as the batch of transitions grows. 
% where the agent estimates the reward rate of policy $\pi$, $r_\pi$, without actually following $\pi$ in the MDP to collect data.
% \citep{precup2001off,sutton2009convergent,sutton2009fast,sutton2016emphatic}.
% where we can draw samples $\qty{(s_k, a_k, r_k, s_k', a_k')}$ from a sampling distribution $d_{\mu\pi}$ at $k$-th iteration.
% We first draw $(s_k, a_k)$ from $d_\mu$,
% where $d_\mu$ is an arbitrary state-action pair distribution satisfying $d_\mu(s, a) > 0 \, \forall s, a$.
% We then get $r_k = r(s_k, a_k)$ and draw $s'_k \sim p(\cdot | s_k, a_k), a'_k \sim \pi(\cdot | s'_k)$.
% Note $d_\mu$ does not necessarily result from a single known behavior policy as \citet{precup2001off,sutton2009convergent,sutton2009fast,sutton2016emphatic},
% it can result from multiple unknown behavior policies,
% for which the setting we consider is usually referred to as behavior-agnostic OPE.

% In this paper, we consider linear function approximation and use $x: \mathcal{S} \times \mathcal{A} \to \R^K$ to denote the feature mapping.
% We use $X \in \R^{\nsa \times K}$ to denote the feature matrix,
% each row of which is $x(s, a)^\top$. Our goal is to approximate $q_\pi + ce$ using $w$ (i.e., $\hat{q} \doteq Xw \approx q_\pi + ce$) and approximate $r_\pi$ using $\hat{r} \in \R$, where $w \in \R^K$ is a learnable weight vector and $\hat r$ is learnable scalar.

Our goal is to approximate, using the data generated from $d_{\mu \pi}$, both the reward rate and the differential value function. The reward rate $r_\pi$ is approximated by a learnable scalar $\hat r$. The differential value function $q_\pi$ is only approximated up to a constant. That is, we are only interested in approximating $q_\pi + c\tb{1}$ for some $c \in R$. This is sufficient if the approximated value function is only used for policy improvement in a control algorithm. 
However, when the state and/or action spaces are large, function approximation may be necessary to represent $q_\pi + c\tb{1}$. 
This paper mainly considers linear function approximation, where the agent is given a feature mapping $x$ that generates a $K$-dimensional vector $x(s, a)$ given a state-action pair $(s, a)$. The agent further maintains a learnable weight vector $w \in \R^K$ and adjusts it to approximate, for all $(s, a)$, $q_\pi(s, a) + c$ using $x(s, a)^\top w$. Let $X \in \R^{\nsa \times K}$ be the feature matrix whose $(s, a)$ row is $x(s, a)^\top$. For the uniqueness of the solution for $w$, it is common to make the following assumption:
\begin{assumption} \label{assu:linearly independent}
$X$ has linearly independent columns. 
\end{assumption}

\section{Differential Semi-Gradient $Q$ Evaluation}
\label{sec:sgq}

% From~\eqref{eq:differential-bellman},
% it is easy to see 
% \begin{align}
% r_\pi = d^\top (r + P_\pi q_\pi - q_\pi),
% \end{align}
% where $d$ can be any state-action pair distribution.
% In particular,
% \begin{align}
% \label{eq:td-error-rate}
% r_\pi = d_\mu^\top (r + P_\pi q_\pi - q_\pi),
% \end{align}
% indicating we can use the error $r + P_\pi q_\pi - q_\pi$ as an estimate of $r_\pi$ \citep{schwartz1993reinforcement,singh1994reinforcement,wan2020learning},
% which requires us to estimate $q_\pi$.

We first present \emph{Differential Semi-gradient $Q$ Evaluation} (Diff-SGQ), 
which is a straightforward extension of the tabular off-policy Differential TD-learning algorithm \citep{wan2020learning} to linear function approximation.
% behavior-agnostic off-policy learning
% and extends the Fitted $Q$ Evaluation in the discounted setting \citep{le2019batch} to the average-reward setting

At the $k$-th iteration, the algorithm draws a sample $(S_k, A_k, R_k, S_k', A_k')$ from $d_{\mu\pi}$ and updates $w_k$ and $\hat{r}_k$ as
\begin{align} 
w_{k+1} &\doteq w_k + \alpha_k \delta_k(w_k, \hat r_k) x_k, \label{eq: Differential FQE 1}\\
\hat{r}_{k+1} &\doteq \hat{r}_k + \alpha_k \delta_k(w_k, \hat r_k) \label{eq: Differential FQE 2},
\end{align}
where $\alpha_k$ is the stepsize used at $k$-th iteration, $x_k \doteq x(S_k, A_k)$, $x_k' \doteq x(S_k', A_k')$, 
% $\zeta > 0$ is some constant, 
and $\delta_k(w, \hat r) \doteq R_k - \hat r + x_k'^\top w - x_k^\top w$ 
% \begin{align}
    % \delta_k(w, \hat r) \doteq R_k - \hat r + x_k'^\top w - x_k^\top w
% \end{align}
is the temporal difference error.
From~\eqref{eq:differential-bellman}, 
it is easy to see $r_\pi = d^\top (r + P_\pi q_\pi - q_\pi)$ holds for any probability distribution $d$; in particular, it holds for $d = d_\mu$,
which is the intuition behind the $\hat{r}$ update~\eqref{eq: Differential FQE 2}.% in Diff-SGQ.
% If there exists a fixed point for \eqref{eq: Differential FQE 1} and \eqref{eq: Differential FQE 2}, denote the fixed point as $(w, \hat r)$, then the fixed point satisfies $\E[\delta_k x_k] = 0$ and $\E[\delta_k] = 0$ simultaneously.
% Equation \eqref{eq: Differential FQE 1} and \eqref{eq: Differential FQE 2} can be combined in the following way. Denote $u_k$ as the concatenation of $\hat r_k$ and $w_k$, $u_k \doteq [\hat{r}_k, w^\top_k]^\top$, $y_k$ as the concatenation of $1$ and $x_k$, $y_k \doteq [1, x^\top_k]^\top$, and $e_1$ be the vector $[1, 0, \cdots, 0]^\top \in \R^{K+1}$. Then \eqref{eq: Differential FQE 1} and \eqref{eq: Differential FQE 2} is equivalent to:
% \begin{align} 
% u_{k+1} &\doteq u_k + \alpha_k \delta_k y_k, \label{eq: Differential FQE 3}\\
% \end{align}
% $Y$ be the concatenation of $\tb{1}$ and $X$, that is, $Y \doteq [\tb{1}, X]$.
% Then \eqref{eq:differential-bellman LFA} can be written as
% % \begin{align}
% % Yu &= \hat{r} \tb{1} + Xw, \\
% % P_\pi Yu &= \hat{r} \tb{1} + P_\pi Xw, \\
% % e_1^\top u \tb{1} &= \hat{r} \tb{1},
% % \end{align}
% % implying
% \begin{align}
%     r - e_1^\top u \tb{1} + P_\pi Yu - Yu = 0
% \end{align}
Diff-SGQ iteratively solves 
\begin{align}
\label{eq: TD fixed point 1}
    \E [\delta_k(w, \hat r)  x_k] = \tb{0} , \qq{and} \E  [\delta_k(w, \hat r) ] = 0,
    % \E  [\delta_k(w, \hat r) ] = 0, \label{eq: TD fixed point 2}
\end{align}
whose solutions, if they exist, are \emph{TD fixed points}. A TD fixed point is an approximate solution to \eqref{eq:differential-bellman} using linear function approximation.  We consider the quality of the approximation in the next section. 
All the proposed algorithms in this paper aim to find a TD fixed point up to some regularization bias if necessary.

In general, there could be no TD fixed point, one TD fixed point, or infinitely many TD fixed points, as in the discounted setting. To see this, let $y_k \doteq [1, x^\top_k]^\top$, $y_k' \doteq [1, x'^\top_k]^\top, u \doteq [\hat{r}, w^\top]^\top$, and $e_1 \doteq [1, 0, \cdots, 0]^\top \in \R^{K+1}$. Then combining \eqref{eq: Differential FQE 1} and \eqref{eq: Differential FQE 2} gives
\begin{align} 
\E[\delta_k(u) y_k] = \tb{0},
\label{eq: TD fixed point variant}
\end{align}
where $\delta_k(u) \doteq R_k - e_1^\top u  + y_k'^\top u - y_k^\top u$.
Writing \eqref{eq: TD fixed point variant} in vector form, we have $Au + b = \tb{0}$,
% \begin{align}
    % A u + b = 0,
% \end{align}
where 
\begin{align}
A &\doteq \E[y_k(-e_1 + y'_k - y_k)^\top]\\ 
& = Y^\top D(P_\pi - I) Y - Y^\top d_\mu e_1^\top \\
& = \mqty[-1 &\tb{1}^\top D(P_\pi - I)X \\-X^\top d_\mu & X^\top D(P_\pi - I)X],  \\
b &\doteq \E[y_k R_k] = Y^\top D r,
Y  \doteq [\tb{1}, X], D \doteq diag(d_\mu) .
\end{align}
If and only if $A$ is invertible, there exists a unique TD fixed point 
\begin{align} \label{eq: TD fixed point form 1}
    u_\text{TD} \doteq - A^{-1} b.
\end{align}
Otherwise, there is either no TD fixed point or there are infinitely many.

% In the on-policy setting,
% where $d_\mu$ is the stationary distribution of $\pi$,
% Lemma 7 in \citet{tsitsiklis1999average} can be used to show $A$ is invertible if, in additional to Assumption \ref{assu:linearly independent}, the following assumption is made:

% Using Schur complement to $A$, we have
% \begin{align}
    % \det (A) = \det(-1) \det (X^\top (D - d_\mu d_\mu^\top)(P_\pi - I)X).
% \end{align}
% Therefore $A$ is invertible if and only if $X^\top (D - d_\mu d_\mu^\top)(P_\pi - I)X$ is invertible, which does not hold in general. However, in the on-policy case, where $d_\mu$ is the stationary distribution of $\pi$, Lemma 7 in \citet{tsitsiklis1999average} shows that $X^\top (D - d_\mu d_\mu^\top)(P_\pi - I)X$ is invertible and thus the TD fixed point uniquely exists if

% In the on-policy case, Diff-SGQ converges to the unique TD fixed point if Assumption \ref{assu:chain}, \ref{assu: positive dmu}, \ref{assu:linearly independent}, \ref{assu:nonconstant}, and \ref{assu: stepsize} hold.~\footnote{Ahbishek Naik, personal communication}~\sz{Not sure if it is convincing enough here. Maybe say "the convergence can be easily established similar to tsitsiklis's paper?"}
Unfortunately, 
even if there exists a unique TD fixed point, 
% In the off-policy case, 
% even if TD fixed points exist,
Diff-SGQ can still diverge,
which exemplifies the deadly triad \citep{sutton2018reinforcement} in the average-reward setting.
The following example
% ,
% inspired by the well-known $\theta \to 2 \theta$ example \citep{sutton2018reinforcement} showing the divergence of off-policy linear TD in the discounted setting,
confirms this point.

\begin{example}[The divergence of Diff-SGQ]

Consider a two-state MDP (Figure \ref{fig:counterexample}). The expected Diff-SGQ update per step can be written as
$
\mqty[\hat r_{k+1} \\ w_{k+1}]  = \mqty[\hat r_{k} \\ w_{k}] + \alpha \left(A \mqty[\hat r_{k} \\ w_{k}] + b\right)
= \mqty[\hat r_{k} \\ w_{k}]  + \alpha \mqty[-1 & 6 \\ -2 & 6] \mqty[\hat r_{k} \\ w_{k}].
$
Here, we consider $\alpha$ a constant stepsize.
% and we show that for all $\alpha > 0$, both two eigen values of the matrix $I + \alpha \mqty[-1 & 6 \\ -2 & 6]$ is greater than 1, which shows divergence of the expected update. 
The eigenvalues of $A = \mqty[-1 & 6 \\ -2 & 6]$ are both positive. 
Hence, no matter what positive stepsize is picked, 
the expected update diverges.
 % because $\mqty[\hat r_k, w_k^\top]^\top$ will go to positive or negative infinity depending on the initial value and 
The sample updates~\eqref{eq: Differential FQE 1} and \eqref{eq: Differential FQE 2} using standard stochastic approximation stepsizes,
therefore,
also diverge. Furthermore, because both eigenvalues are positive, A is an invertible matrix, implying the unique existence of the TD fixed-point. 

\begin{figure}[tbh]
\centering
\includegraphics[width=0.8\linewidth]{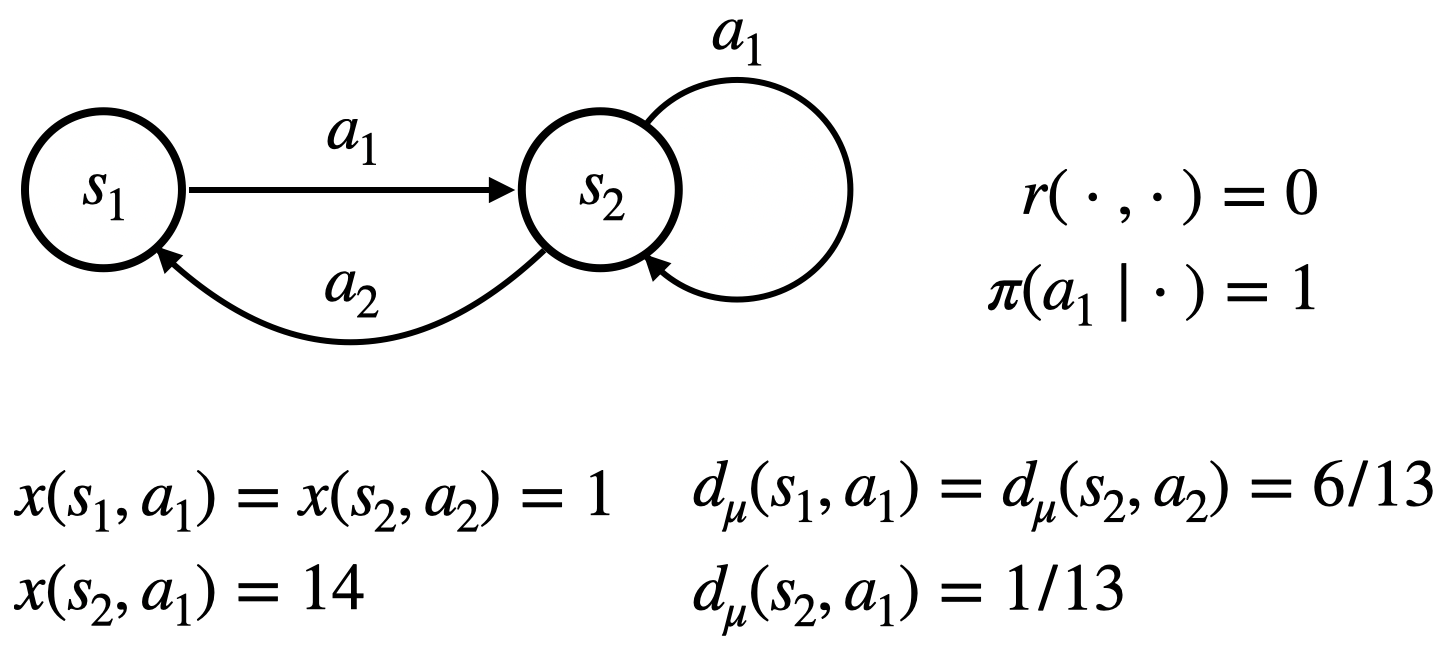}
\caption{\label{fig:counterexample}
An example showing the divergence of Diff-SGQ.
}
\end{figure}

\end{example}

% Motived by the celebrated success of the GTD algorithms in addressing the deadly triad in the discounted setting,
% Instead of using semi-gradients like Diff-SGQ,
% we propose and analyze two novel convergent gradient-based methods in the following two sections.
% for behavior-agnostic OPE in the average-reward setting with linear function approximation.

\section{One-Stage Differential Gradient $Q$ Evaluation}
We now present an algorithm that is guaranteed to converge to the TD fixed point $\eqref{eq: TD fixed point variant}$ if it uniquely exists.
% We now propose an algorithm to find TD fixed point \eqref{eq: TD fixed point variant}.
Motivated by the Mean Squared Projected Bellman Error (MSPBE) defined in the discounted setting and used by Gradient TD algorithms, we define the MSPBE in the average-reward setting as
\begin{align} \label{eq: MSPBE vector}
\text{MSPBE}_1(u) &= \norm{\Pi_Y \bar{\delta}(u)}^2_D,
\end{align}
where $\Pi_Y \doteq Y(Y^\top DY)^{-1}Y^\top D$ is the projection matrix and $\bar{\delta}(u) \doteq r - e_1^\top u \tb{1} + P_\pi Yu - Yu$
is the vector of TD errors for all state-action pairs.
The vector $\Pi_Y \bar \delta(u)$ is the projection of the vector of TD errors on the column space of $Y$. 
% The on-policy version of this MSPBE objective also appears in \citet{abbasi2019politex}.
% And the MSPBE is the squared $D$ norm of the projection.
The existence of the matrix inverse in $\Pi_Y$, $(Y^\top DY)^{-1}$, is guaranteed by Assumption \ref{assu: positive dmu} and 
\begin{assumption}
\label{assu:nonconstant}
For any $w \in \R^{K}$ such that $w \neq \tb{0}$ and for any $c \in \R$, 
$Xw \neq c\tb{1}$.
\end{assumption}
The above assumption guarantees that if $w^*$ is a solution for $w$ in \eqref{eq: TD fixed point 1}, then no other solution's approximated action-value function would be identical to $Xw^*$ up to a constant. 
This assumption is also used by \citet{tsitsiklis1999average} in their on-policy policy evaluation algorithms in average-reward MDPs. 
Apparently the assumption does not hold in the tabular setting (i.e., when $X = I$).
 % because there are infinite many solutions for $q$ and they only differ by some constants.
However, with function approximation, we usually have many more states than features (i.e., $\ns \gg K$),
in which case the above assumption would not be restrictive.

Let $C \doteq Y^\top D Y$, we have $\Pi^\top D \Pi = DY C^{-1} Y^\top D$, with which we give a different form for \eqref{eq: MSPBE vector}:
\begin{align}
\label{eq: MSPBE} 
    \text{MSPBE}_1(u) &= \norm{Y^\top D \bar{\delta}(u)}^2_{C^{-1}} = \norm{Au + b}^2_{C^{-1}}\\
    &= \E[\delta_k(u) y_k]^\top \E[y_k y_k^\top]^{-1} \E[\delta_k(u) y_k].
\end{align}
% \sz{If we use the term ``MSPBE'', it may be better to explicitly have some ``projection'' in the objective.}
% where
% \begin{align}
% \delta_k(u) &\doteq r_k - e_1^\top u  + y_k'^\top u - y_k^\top u, \\
% \bar{\delta}(u) &\doteq r - e_1^\top u \tb{1} + P_\pi Yu - Yu,
% \end{align}
% $C \doteq Y^\top D Y$,
% $\Pi_Y \doteq Y(Y^\top DY)^{-1}Y^\top D$ is the projection to the column space of $Y$ w.r.t. $\norm{\cdot}_D$,
% \yw{r vector and r function share the same notation r}
% \sz{We use function and vector interchangeably as mentioned in the beginning of background.}

It can be seen 
% \sz{It cannot be easily seen unless we have the vector form. So I moved them here.} 
that if \eqref{eq: TD fixed point variant} has a solution, 
then that solution also minimizes \eqref{eq: MSPBE}, 
in which case solving \eqref{eq: TD fixed point variant} can be converted to minimizing \eqref{eq: MSPBE}.
% \begin{align}
% \norm{\Pi_Y (r - e_1^\top u \tb{1} + P_\pi Yu - Yu)}_D^2,
% \end{align}
% where $D \in \R^{\nsa \times \nsa}$ is the diagonal matrix whose diagonal entry is $d_\mu$, $\Pi_Y$ is a projection mapping a vector in $\R^\nsa$ to the column space of $Y$ w.r.t. the norm $\norm{\cdot}_D$, i.e.,
% \begin{align}
% \Pi_Y = Y(Y^\top D Y)^{-1}Y^\top D.
% \end{align}
% However, as we have shown in Section \ref{sec:td fixed point}, \eqref{eq: TD fixed point variant} may not possess a solution. 
% \yw{What is the actual reason? Is it because our algorithm can not minimize 7 or minimizer of 7 is worse than minimizer of 8? }
% \sz{See below}
However, when \eqref{eq: TD fixed point variant} does not have a unique solution, the set of minimizers of \eqref{eq: MSPBE} could be unbounded and thus algorithms minimizing $\text{MSPBE}_1$ risk generating unbounded updates.
To ensure the stability of our algorithm when \eqref{eq: TD fixed point variant} does not have a unique solution,
we use a regularized $\text{MSPBE}_1$ as our objective:
% Our algorithm doesn't minimize the MSPBE, instead, it minimizes a slightly different objective:
\begin{align}
    J_{1, \eta}(u) \doteq \norm{Au + b}_{C^{-1}}^2 + \eta u^\top I_0 u \label{eq: MSPBE + reg},
\end{align}
where $I_0 \doteq diag(\tb{1} - e_1)$, $\eta$ is a positive scalar, and $\eta u^\top I_0 u = \eta \norm{w}^2$ is a ridge regularization term on $w$.
% As we will show later, this new objective can be minimized for any $\eta > 0$ and therefore can get arbitrarily close to MSPBE. 

To minimize $J_{1, \eta}(u)$, 
one could proceed with techniques used in TDC \citep{sutton2009fast},
which we leave for future work. 
In this paper,
we proceed with the saddle-point formulation of GTD2 introduced by \citet{liu2015finite},
which exploits
% We first write \eqref{eq: MSPBE + reg} in the vector form:
% \begin{align}
    % J_{1, \eta}(u) = \norm{Y^\top D \delta(u)}_{C^{-1}}^2 + \eta u^\top I_0 u
% \end{align}
% we arrive at our final objective 
% \begin{align}
% J_{1, \eta}(u) \doteq \norm{\Pi_Y \delta(u)}_D^2 + \eta u^\top I_0 u,
% \end{align}
% where $\delta(u) \doteq r - e_1^\top u \tb{1} + P_\pi Yu - Yu$, 
% $C \doteq Y^\top D Y$,
% \remove{$r$ is a $|\mathcal{S}| \times |\mathcal{A}|$-size vector with its $(s, a)$ element being $r(s, a)$.} 
% ,
% and $I_0 \doteq diag(\tb{1} - e_1)$ \yw{we need to choose to use bold symbols or plain symbols for vectors.}.
% The saddle-point formulation uses 
Fenchel's duality:
\begin{align}
u^\top M^{-1} u = \max_\nu 2 u^\top \nu - \nu^\top M \nu,
\end{align}
for any positive definite $M$, yielding
% If Assumption \ref{assu: positive dmu}, \ref{assu:linearly independent}, and \ref{assu:nonconstant} hold, $C$ is positive definite. 
% Applying Fenchel's duality to $J_{1, \eta}$ yields
\begin{align}
\label{eq:gq1-saddlepoint}
& J_{1, \eta}(u) \\
 = &\textstyle{\max_{\nu \in \R^{K+1}} 2 \nu^\top Y^\top D \bar{\delta}(u) - \nu^\top C \nu + \eta u^\top I_0 u}.
\end{align}
% where 
% the first equality is due to the identity $\Pi_Y^\top D \Pi_Y = DYC^{-1} Y^\top D$ and 
% the second equality is 
% due to Fenchel's duality
% \begin{align}
% u^\top M^{-1} u = \max_\nu 2 u^\top \nu - \nu^\top M \nu,
% \end{align}
% which holds for any positive definite $M$.
% We will show later that under mild conditions $C$ is positive definite.
So 
% \begin{align}
$\min_u J_{1, \eta}(u) = \min_u \max_\nu J_{1, \eta}(u, \nu)$,
% \end{align}
where 
\begin{align}
J_{1, \eta}(u, \nu) \doteq 2 \nu^\top Y^\top D \bar{\delta}(u) - \nu^\top C \nu + \eta u^\top I_0 u.
\end{align}
As $J_{1, \eta}(u, \nu)$ is convex in $u$ and concave in $\nu$,
we have now reduced the problem into a convex-concave saddle point 
% (CCSP) 
problem.
Applying primal-dual methods to this
% CCSP 
problem, 
that is,
performing gradient ascent for $\nu$ following $\nabla_\nu J_{1, \eta}(u, \nu)$ and gradient descent for $u$ following $\nabla_u J_{1, \eta}(u, \nu)$,
we arrive at our first new algorithm, \emph{One-Stage 
Differential Gradient $Q$ Evaluation}, or \emph{Diff-GQ1}. 
At the $k$-th iteration,
with a sample $(S_k, A_k, R_k, S_k', A_k')$ from $d_{\mu\pi}$,
Diff-GQ1 updates $u_k$ and $\nu_k$ as
\begin{align}
\label{eq:gq1-update}
\delta_k &\doteq R_k - e_1^\top u_k + y_k'^\top u_k - y_k^\top u_k, \\
\nu_{k+1} &\doteq \nu_k + \alpha_k (\delta_k - y_k^\top \nu_k) y_k, \\
u_{k+1} &\doteq u_k + \alpha_k(y_k - y_k' + e_1)y_k^\top \nu_k - \alpha_k \eta I_0 u_k,
\end{align}
where $\qty{\alpha_k}$ is the sequence of learning rates satisfying the following standard assumption:
\begin{assumption} \label{assu: stepsize}
$\{\alpha_k\}$ is a positive deterministic nonincreasing sequence s.t. $\sum_k \alpha_k = \infty$ and $\sum_k \alpha_k^2 < \infty$.
\end{assumption}
The algorithm is one-stage because, while there are two weight vectors updated in every iteration, 
both converge simultaneously.
% the convergence of each vector does not wait for that of the other vector.
% is a deterministic non-increasing sequence of learning rates satisfying $\sum_k \alpha_k = \infty$ and $\sum_k \alpha_k^2 < \infty$,
% $y_k \doteq [1, x(s_k, a_k)^\top]^\top$,
% $y_k' \doteq [1, x(s_k', a_k')^\top]^\top$.
% To study the convergence of Diff-GQ1, 
% we make the following assumptions:
% \begin{assumption}
% \label{assu:gq1-iid}
% $\qty{(s_k, a_k, r_k, s_k', a_k')}$ are i.i.d. samples from $d_{\mu\pi}$ and $x(s_k, a_k), x(s_k', a_k')$ have uniformly bounded second moments.
% \end{assumption}
% We have
\begin{theorem}
\label{thm:gq1-convergence}
If Assumptions \ref{assu:chain}, \ref{assu: positive dmu}, \ref{assu:nonconstant},  \& \ref{assu: stepsize} hold, then for any $\eta > 0$, almost surely,
the iterate $\qty{u_k}$ generated by Diff-GQ1~\eqref{eq:gq1-update} converges to $u_\eta^*$,
% \begin{align}
% \lim_{k \to \infty} u_k = u^*_\eta,
% \end{align}
where $u^*_\eta \doteq -(\eta I_0 + A^\top C^{-1} A)^{-1}A^\top C^{-1}b$ is the unique minimizer of $J_{1, \eta}(u)$.
Further, if $A$ is invertible, then for $\eta = 0$, $\qty{u_k}$ converges almost surely to the $u_\text{TD}$ defined in \eqref{eq: TD fixed point form 1}.
\end{theorem}
We defer the full proof to Section~\ref{sec:proof-gq1}. 
\begin{proof}
(Sketch)
With $\kappa_k \doteq [\nu_k^\top, u_k^\top]^\top$,
we rewrite~\eqref{eq:gq1-update} as
\begin{align}
\kappa_{k+1} \doteq \kappa_k + \alpha_k (G_{k+1} \kappa_k + h_{k+1}),
\end{align}
where 
\begin{align}
G_{k+1} &\doteq \mqty[-y_k y_k^\top & y_k(y_k' - y_k)^\top - y_k e_1^\top \\
(y_k - y_k')y_k^\top  + e_1 y_k^\top & -\eta I_0], \\
h_{k+1} &\doteq \mqty[y_k r_k \\ \tb{0}].
\end{align}
The asymptotic behavior of $\qty{\kappa_k}$ is governed by
\begin{align}
\bar{G} &\doteq \E[G_{k+1}] = \mqty[-C & A \\ -A^\top & - \eta I_0], \bar{h} \doteq \E[h_{k+1}] = \mqty[b \\ \tb{0} ].
\end{align} 
% where 
% \begin{align}
% A &\doteq Y^\top D(P_\pi - I) Y - Y^\top d_\mu e_1^\top \\
% &= \mqty[-1 &\tb{1}^\top D(P_\pi - I)X \\-X^\top d_\mu & X^\top D(P_\pi - I)X],  \\
% b &\doteq Y^\top D r.
% \end{align}
The convergence of $\kappa_t$ to a unique point can be guaranteed if $\bar G$ is a Hurwitz matrix, or equivalently, if the real part of any eigenvalue of $\bar{G}$ is strictly negative.
Therefore, it is important to first ensure that $\bar{G}$ is nonsingular.
If $A$ was nonsingular,
we can show $\bar{G}$ being nonsingular easily even with $\eta = 0$.
% Readers familiar with GTD2 \citep{sutton2009fast} may find their $A$ matrix $X^\top D(\gamma P_\pi - I)X$ is assumed to be nonsingular in their discounted setting.
% However,
% in our average-reward setting without the discount factor $\gamma < 1$,
% it becomes restrictive to assume $A$ is nonsingular.
However, in general, $A$ may not be nonsingular and therefore, we require $\eta > 0$ to ensure $\bar{G}$ being nonsingular.
% That being said, 
% we do not need $\eta > 0$ if our $A$ is nonsingular and $\eta$ can be 0.
We can easily show that the real part of any eigenvalue of $\bar{G}$ is strictly negative and thus $\bar G$ is Hurwitz.
Standard stochastic approximation results \citep{borkar2009stochastic} then show $\textstyle{\lim_k \kappa_k = -\bar{G}^{-1} \bar{h}}.$
% \begin{align}
% \textstyle{\lim_k \kappa_k = -\bar{G}^{-1} \bar{h}}.
% \end{align}
Define $u^*_\eta$ as the lower half of $-\bar{G}^{-1}\bar{h}$,
we have $u^*_\eta = -(\eta I_0 + A^\top C^{-1} A)^{-1}A^\top C^{-1}b.$
% \begin{align}
% u^*_\eta = -(\eta I_0 + A^\top C^{-1} A)^{-1}A^\top C^{-1}b.
% \end{align} 
It is easy to verify (e.g., using the first order optimality condition of $J_{1, \eta}(u)$) that $u^*_\eta$ is the unique minimizer of $J_{1, \eta}(u)$.
% Further, when $A$ is invertible, $\eta$ can be 0 and $\bar G$ is still Hurwitz. Therefore, $u_k$ converges to $-A^{-1}b$ almost surely.
\end{proof}
% To further understand Diff-GQ1,
% we rewrite~\eqref{eq:gq1-update} as
% \begin{align}
% \delta_k &\doteq r_k - \hat{r}_k + x_k'^\top w_k - x_k^\top w_k, \\
% \nu_{k+1} &\doteq \nu_k + \alpha_k (\delta_k - y_k^\top \nu_k) y_k,\\
% \hat{r}_{k+1} &\doteq \hat{r}_k + \alpha_k y_k^\top \nu_k - \alpha_k \eta \hat{r}_k, \\
% w_{k+1} &\doteq w_k + \alpha_k(x_k - x_k')y_k^\top \nu_k - \alpha_k \eta u_k.
% \end{align}
% % with both the original features $x_k \doteq x(s_k, a_k), x_k' \doteq x(s_k', a_k')$ and the augmented feature $y_k$.
% It then becomes clear that the auxiliary variable $\nu$,
% used to regress the TD error $\delta$,
% runs on the augmented feature $y_k$,
% while the main variable $w$ runs on the original feature $x_k$ directly. \yw{Don't quite understand this part.}

\textbf{Quality of TD Fixed Points.}
We now analyze the quality of TD fixed points. For our analysis, we make the following assumption.
\begin{assumption}\label{assu: existence of TD fixed points}
There exists at least one TD fixed point.
\end{assumption}
Let $u^* = [\hat{r}^*, w^{*\top}]^\top$ be one fixed point (a solution of \eqref{eq: TD fixed point variant}). We are interested in the upper bound of the absolute value of the difference between the estimated reward rate and the true reward rate $|\hat{r}^* - r_\pi|$ and also the upper bound of the minimum distance between the estimated differential value function to the set $\qty{q_\pi + c\tb1}$. 
In general,
as long as there is representation error,
the TD fixed point can be arbitrarily poor in terms of approximating the value function, 
even in the discounted case (see \citet{kolter2011fixed} for more discussion).
In light of this,
we study the bounds only when $d_\mu$ is close to $d_\pi$, 
the stationary state-action distribution of $\pi$,
in the sense of the following assumption.
Let $\xi \in (0, 1)$ be a constant, 
% we assume
\begin{assumption}
\label{assu:mupi} $F$ is positive semidefinite, where
\begin{align}
F \doteq \mqty[X^\top D X & X^\top DP_\pi X \\ X^\top P_\pi^\top D X & \xi^2 X^\top D X].
\end{align}
\end{assumption}
A similar assumption about $F$ is also used by \citet{kolter2011fixed}
in the analysis of the performance of the MSPBE minimizer in the discounted setting.
\citet{kolter2011fixed} uses $\xi = 1$ while we use $\xi < 1$ to account for the lack of discounting.
In Section~\ref{sec:assumption-simulation},
we show with simulation that this assumption holds with reasonable probability in our randomly generated MDPs.
Furthermore, we consider the bounds when all the features have zero mean under the distribution $d_\mu$. 
\begin{assumption}\label{assu: 0 mean feature vector}
$X^\top d_\mu = \tb{0}$.
\end{assumption}
This can easily be done by subtracting each feature vector sampled in our learning algorithm by some estimated mean feature vector, which is the empirical average of all the feature vectors sampled from $d_{\mu}$. 
Note without this mean-centered feature assumption,
a looser bound can also be obtained.
Our intention here is to show that bounds of our algorithms are on par with their counterparts in the discounted setting and thus one does not lose these bounds when one moves from the discounted setting to the average-reward setting.
% Mean centered features, along with Assumption \ref{assu:mupi}, guarantees the unique existence of the TD fixed point. Moreover, it permits a tighter bound for TD-fixed point (a looser bound without centering features is also provided in Section~\ref{sec:gq1-model-error}).

\begin{proposition}
\label{prop:gq1-model-error}
Under Assumptions \ref{assu:chain}, \ref{assu: positive dmu}, \ref{assu:nonconstant}, \ref{assu: existence of TD fixed points} -- \ref{assu: 0 mean feature vector},
\begin{align}
% \inf_{c \in \R}\norm{q_\pi^c - Xw_*}_D &\leq \frac{(\norm{P_\pi}_D + 1)}{1 - \xi} \inf_{c \in \R} \norm{\Pi_X q_\pi^c - q_\pi^c}_D,
% \end{align}
% \begin{align}
% |r_\pi - \hat{r}^*| & \leq (\norm{P_\pi}_D + 1) \inf_{c \in \R} \norm{\Pi_X q_\pi^c - q_\pi^c}_D,
&\inf_{c \in \R} \norm{Xw^* - q_\pi^c}_D \leq  \frac{\norm{P_\pi}_D + 1}{1 - \xi} \inf_{c \in \R} \norm{\Pi_X q_\pi^c - q_\pi^c}_D, \\
&|r_\pi - \hat r^*| \\
& \leq \textstyle{\frac{\norm{d_\mu^\top (P_\pi - I)}_{D^{-1}} (\norm{P_\pi}_D + 1)}{1 - \xi}} \inf_{c \in \R} \norm{\Pi_X q_\pi^c - q_\pi^c}_D,
\end{align}
where $q_\pi^c \doteq q_\pi + c \tb{1}$ and $\Pi_X = X(X^\top D X)^{-1}X^\top D$.
\end{proposition}
% The above bound can be tightened, if the agent ensures that its features have zero mean under the distribution $\mu$, that is, $X^\top d_\mu = \tb{0}$. This can be done by subtracting the feature vector by an estimated mean feature vector.
% \begin{corollary}
% \label{corr:gq1-model-error}
% Under Assumptions~\ref{assu:nonconstant},~\ref{assu:Existence of TD fixed point}, \& ~\ref{assu:mupi}, if further, $X^\top d_\mu = \tb{0}$, then
% \begin{align}
% &\inf_{c \in \R}\norm{q_\pi - Xw_*}_D \leq \frac{(\norm{P_\pi}_D + 1)}{1 - \xi} \inf_{c \in \R} \norm{\Pi_X q_\pi^c - q_\pi^c}_D,
% \end{align}
% \begin{align}
% &|r_\pi - \hat{r}^*| \leq (\norm{P_\pi}_D + 1)\inf_{c \in \R} \norm{\Pi_X q_\pi^c - q_\pi^c}_D,
% \end{align}
% where $q_\pi^c \doteq q_\pi + c \tb{1}$ and $\Pi_X = X(X^\top D X)^{-1}X^\top D$.
% \end{corollary}
We defer the proof to Section~\ref{sec:gq1-model-error}. 
As a special case, there exists a unique TD fixed point in the on-policy case (i.e., $d_\mu = d_\pi$) under Assumptions~\ref{assu:chain},~\ref{assu:linearly independent}, and~\ref{assu:nonconstant}. 
Then $|r_\pi - \hat r^*| = 0$ as $d_\pi^\top (P_\pi - I) = \tb{0}$ and a tighter bound for the estimated differential value function can be obtained. See \citet{tsitsiklis1999average} for details. 
% \yw{1. It would be good if we can show that this bound in the on-policy case matches with the average cost TD's bound. 2. Can we have a bound for $q_w$?}

\textbf{Finite Sample Analysis.}
We now provide finite sample analysis for a variant of Diff-GQ1, \emph{Projected Diff-GQ1},
which is detailed in Section~\ref{sec:pgq1} in the appendix.
Projected Diff-GQ1 is different from Diff-GQ1 in three ways: 1) for each iteration, Projected Diff-GQ1 projects the two updated weight vectors to two bounded closed sets to ensure that the weight vectors do not become too large, 2) Projected Diff-GQ1 uses a constant stepsize, and 3) Projected Diff-GQ1 does not impose ridge regularization,
that is,
it considers the objective $\text{MSPBE}_1$ directly.
\begin{proposition}
(Informal) Under standard assumptions,
if Assumption~\ref{assu:mupi} holds and $A$ is nonsingular, 
with proper stepsizes, 
with high probability,
the iterates $\qty{\hat r_k}, \qty{w_k}$ generated by Projected Diff-GQ1 satisfy
\begin{align}
&(\bar{\hat r}_k - r_\pi)^2 = \textstyle{\mathcal{O}(\inf_{c \in \R}\norm{X \bar{w}_k - q_\pi^c }^2)} \\
=& \textstyle{\mathcal{O}(k^{-\frac{1}{2}})  + \mathcal{O}(\inf_{c \in \R} \norm{\Pi_X q_\pi^c - q_\pi^c}_D^2)},
\end{align}
where $\bar{\hat r}_k \doteq (1 / k) \sum_{i=1}^{k} \hat{r}_i, \bar{w}_k \doteq (1 / k) \sum_{i=1}^{k} w_i$.
\end{proposition}
We defer the precise statement and its proof to Section~\ref{sec:pgq1}.

\section{Two-Stage Differential Gradient $Q$ Evaluation}
While Assumption~\ref{assu:nonconstant} is not restrictive, we present in this section a new algorithm that does not require it but can still converge to the TD fixed point if it uniquely exists. 
The algorithm achieves this by drawing one more sample from $d_{\mu\pi}$ for each iteration, and performing two learning stages, where $\hat{r}$ converges only when $w$ has converged.
 % the first stage converges without waiting for the convergence of the second stage, but the second stage converges only after the first stage converges. 
% The algorithm achieves this by drawing one more sample from $d_{\mu\pi}$ and making one more update per iteration, \sz{Is there one more update? $u \in \R^{K+1}$ but $w \in \R^{K}$} compared with Diff-GQ1. 
We call this algorithm \emph{Two-Stage Differential Gradient $Q$ Evaluation}, or \emph{Diff-GQ2}, and derive it as follows.

% Assumption~\ref{assu:gq1-nonconstant} is required because otherwise, if there is a solution for $w$ then any $X(w + w')$ is also a solution
% we present in this section Two Stage Differential Gradient $Q$ Evaluation, or Diff-GQ2.

% \citet{schwartz1993reinforcement,singh1994reinforcement,wan2020learning} and~\eqref{eq:td-error-rate} suggest that 
Consider the TD fixed point \eqref{eq: TD fixed point 1}. Writing $\E[\delta_k(w, \hat r)] = 0$ in vector form, we have
\begin{align}
    \hat r = d_\mu^\top (r + P_\pi Xw - Xw). \label{eq: TD fixed point 2 vector}
\end{align}
Replacing $\hat r$ in $\E [\delta_k(w, \hat r)  x_k] = \tb{0}$ with \eqref{eq: TD fixed point 2 vector}, we have:
\begin{align}
    & X^\top D (r + P_\pi Xw - Xw) \\
    & - X^\top D \tb{1} d_\mu^\top (r + P_\pi X w - X w) = \tb{0},
\end{align}
or equivalently
\begin{align} \label{eq: TD fixed point 1 vector}
    A_2 w + b_2 = \tb{0},
\end{align}
where $A_2 \doteq X^\top (D - d_\mu d_\mu^\top)(P_\pi - I)X$, $b_2  \doteq X^\top (D - d_\mu d_\mu^\top) r$.
% \begin{align} 
% A_2 & \doteq X^\top (D - d_\mu d_\mu^\top)(P_\pi - I)X , \\
% b_2 & \doteq X^\top (D - d_\mu d_\mu^\top) r.
% \end{align}
The combination of \eqref{eq: TD fixed point 2 vector} and \eqref{eq: TD fixed point 1 vector} is an alternative definition for TD fixed points. When $A_2$ is invertible, 
the unique TD fixed points are 
\begin{align} \label{eq: TD fixed point form 2}
    w_\text{TD} &\doteq -A_2^{-1} b_2,\\
    \hat r_\text{TD} &\doteq d_\mu^\top (r + P_\pi Xw_\text{TD} - Xw_\text{TD}).
\end{align}
It is easy to verify that $u_\text{TD} = [\hat r_\text{TD}, w_\text{TD}^\top]^\top$, where $u_\text{TD}$ is defined in \eqref{eq: TD fixed point form 1}.
% When he unique TD fixed point is $w_{TD} = -A_2^{-1}b_2, \hat r_{TD} = d_\mu^\top(r + P_\pi X w - Xw)$   and
% Therefore, instead of maintaining a separate scalar estimate $\hat{r}$ of $\bar r_\pi$ as in Diff-GQ1, we directly use $d_\mu^\top (r + P_\pi Xw - Xw)$ as an estimate of it. The TD 
% in place of $\bar{r}$ in the differential Bellman equation~\eqref{eq:differential-bellman},
% yielding the following 

Denote $\bar{r}_w \doteq r + P_\pi Xw - Xw$, then \eqref{eq: TD fixed point 1 vector} can be written as $X^\top D ( \bar{r}_w - d_\mu^\top \bar{r}_w \tb{1}) = 0$, from which we define a new MSPBE objective:
\begin{align}
\text{MSPBE}_2(w) \doteq \norm{\Pi_X (\bar{r}_w - d_\mu^\top \bar{r}_w \tb{1}))}_D^2,
% \norm{X^\top D ( \bar{r}_w - d_\mu^\top \bar{r}_w \tb{1})}_{C_2^{-1}}^2,
\end{align}
where $C_2 \doteq X^\top D X$ in $\Pi_X$ is invertible under Assumption \ref{assu: positive dmu}. $\text{MSPBE}_2$ is different from $\text{MSPBE}_1$ defined in \eqref{eq: MSPBE} in that $\text{MSPBE}_2$ is a function of $w$ only while $\text{MSPBE}_1$ is a function of both $w$ and $\hat r$. However, the solutions of MSPBE$_2(w)$ = 0 are exactly the solutions of $w$ in $\text{MSPBE}_1(u)$ = 0, if both solutions exist.

After introducing a ridge term with $\eta > 0$ for the same reason as Diff-GQ1,
we arrive at the objective that Diff-GQ2 minimizes:
\begin{align}
J_{2, \eta}(w) \doteq \norm{X^\top D ( \bar{r}_w - d_\mu^\top \bar{r}_w \tb{1})}_{C^{-1}_2}^2 + \eta \norm{w}^2.
\end{align}

Applying Fenchel's duality on $J_{2, \eta}(w)$ yields 
% \begin{align}
$\textstyle{\min_w J_{2, \eta}(w) = \min_w \max_\nu J_{2, \eta} (w, \nu)}$,
% \end{align}
where 
\begin{align}
&J_{2, \eta}(w, \nu) \\
\doteq& 2 \nu^\top X^\top D (\bar{r}_w - d_\mu^\top \bar{r}_w \tb{1}) - \nu^\top C_2 \nu + \eta \norm{w}^2.
\end{align}
$J_{2, \eta}(w, \nu)$ is convex in $w$ and concave in $\nu$.
To apply primal-dual methods for finding the saddle point of $J_{2, \eta}(w, \nu)$,
we need to obtain unbiased samples of $X^\top D (\bar{r}_w - d_\mu^\top \bar{r}_w \tb{1})$.
As this term includes two nested expectations (i.e., $D$ and $d_\mu$),
Diff-GQ2 requires two i.i.d. samples $(S_{k,1}, A_{k, 1}, R_{k, 1}, S_{k, 1}', A_{k, 1}')$ and $(S_{k,2}, A_{k, 2}, R_{k, 2}, S_{k, 2}', A_{k, 2}')$ from $d_{\mu\pi}$ at the $k$-th iteration for a single gradient update.
This is not the notorious double sampling issue in minimizing the Mean Square Bellman Error (see, e.g., \citet{baird1995residual} and Section 11.5 by \citet{sutton2018reinforcement}),
where two successor states $s_1'$ and $s_2'$ from a single state action pair $(s, a)$ are required,
which is not possible in the function approximation setting. 
% usually impractical without having access to the transition kernel $p$.
Sampling two i.i.d. tuples from $d_{\mu\pi}$ is completely feasible.

At the $k$-th iteration, Diff-GQ2 updates $\nu$ and $w$ as
\begin{align}
\label{eq:gq2-update-uw}
\nu_{k+1} &\doteq \nu_k + \alpha_k \Big(R_{k, 1} + x_{k, 1}'^\top w_k - x_{k, 1}^\top w_k \\
&\quad - (R_{k, 2} + x_{k, 2}'^\top w_k - x_{k, 2}^\top w_k) - x_{k, 1}^\top \nu_t \Big) x_{k, 1}, \\
w_{k+1} &\doteq w_k + \alpha_k \Big(x_{k, 1} - x_{k, 1}' + (x_{k, 2} - x_{k, 2}') \Big) x_{k, 1}^\top \nu_k \\
&\quad - \alpha_k \eta w_k,
\end{align}
% (Equation \ref{eq: Differential FQE 1}-\re)
where $x_{k, i} \doteq x(S_{k, i}, A_{k, i}), x_{k, i}' \doteq x(S_{k, i}', A_{k, i}'), \qty{\alpha_k}$, again, satisfies Assumption \ref{assu: stepsize}.
Additionally,
following~\eqref{eq: TD fixed point 2 vector}, Diff-GQ2 updates $\hat{r}$ as
\begin{align}
\hat{r}_{k+1} \doteq \hat{r}_k + \beta_k \Big(& \textstyle{\frac{1}{2}\sum_{i=1}^2} (R_{k, i} + x_{k, i}'^\top w_k - x_{k, i}^\top w_k) \\&- \hat{r}_k \Big),
\label{eq:gq2-update-r}
\end{align}
where $\qty{\beta_k}$ satisfies the same assumption as $\{\alpha_k\}$.
\begin{assumption} \label{assu: stepsize 2}
$\{\beta_k\}$ is a positive deterministic nonincreasing sequence s.t. $\sum_k \beta_k = \infty$ and $\sum_k \beta_k^2 < \infty$.
\end{assumption}
% In Diff-GQ2 updates~\eqref{eq:gq2-update-uw} \&~\eqref{eq:gq2-update-r},
% $\hat{r}$ depends on $w$ but $w$ does not depend on $\hat{r}$.
% So though $\hat{r}$ can possibly improve during the learning of $w$,
% $\hat{r}$ converges only if $w$ has converged.
% For this reason, 
% Diff-GQ2 is a two stage algorithm (c.f. bi-level optimization) instead of a two time-scale algorithm.
% ,and thus we do not need to pose assumptions on $\alpha_k / \beta_k$.
% To study the convergence of Diff-GQ2,
% we assume
% We have
% \begin{assumption}
% \label{assu:gq2-full-rank}
% $X$ has linearly independent columns.
% \end{assumption}
% \begin{assumption}
% \label{assu:gq2-iid}
% For $i=1, 2$, \\$\qty{(s_{k,i}, a_{k, i}, r_{k, i}, s_{k, i}', a_{k, i}')}$ 
% are i.i.d. samples from $d_{\mu\pi}$ and $x(s_{k,i}, a_{k, i}), x(s_{k, i}', a_{k, i}')$ have uniformly bounded second moments.
% \end{assumption}

% Assumption~\ref{assu:gq2-full-rank} is weaker than Assumption~\ref{assu:gq1-nonconstant} and is standard in policy evaluation with linear function approximation \citep{tsitsiklis1997analysis,sutton2009convergent,sutton2009fast,zhang2020gradientdice}. 
% Let ${A_2} \doteq X^\top (D - d_\mu d_\mu^\top)(P_\pi - I)X , b_2 \doteq X^\top (D - d_\mu d_\mu^\top) r$, we have
\begin{theorem}
\label{thm:gq2-convergence}
If Assumptions~\ref{assu:chain}, \ref{assu: positive dmu}, \ref{assu:linearly independent}, \ref{assu: stepsize}, \& \ref{assu: stepsize 2} hold, then
almost surely,
the iterates $\qty{w_k}, \qty{\hat{r}_k}$ generated by Diff-GQ2~\eqref{eq:gq2-update-uw} \&~\eqref{eq:gq2-update-r} satisfy
\begin{align}
{\lim_{k \to \infty}} w_k &= w^*_\eta,
{\lim_{k \to \infty}} \hat{r}_k &= d_\mu^\top (r + P_\pi X w^*_\eta - X w^*_\eta),
\end{align}
where $w^*_\eta \doteq -(\eta I + A_2^\top C_2^{-1} A_2)^{-1}A_2^\top C_2^{-1}b_2$ is the unique minimizer of $J_{2, \eta}(w)$. 
Define
$w^*_0 \doteq \lim_{\eta \downarrow 0} w^*_\eta $,
we have
\begin{align}
\textstyle \norm{w^*_\eta - w^*_0} \leq \eta U_0
\end{align}
for some constant $U_0$.
Further, if Assumption~\ref{assu: existence of TD fixed points} holds, then
$A_2 w^*_0 + b_2 = 0$,
% $\lim_{\eta \downarrow 0}w^*_\eta =- \left( C_2^{-\frac{1}{2}} A\right)^\dagger C_2^{-\frac{1}{2}} b_2$, where $(\cdot)^\dagger$ denotes the Moore–Penrose inverse.
and if $A_2$ is invertible, then for $\eta = 0$, $w_k$ and $\hat r_k$ converge almost surely to $w_\text{TD}$ and $\hat r_\text{TD}$ defined in \eqref{eq: TD fixed point form 2}.
\end{theorem}
We defer the full proof to Section~\ref{sec:proof-gq2}. Similar to Projected Diff-GQ1,
we provide a finite sample analysis for a variant of Diff-GQ2,
\emph{Projected Diff-GQ2},
in Section~\ref{sec:pgq2}.

\begin{figure*}[t]
\centering
\includegraphics[width=0.8\textwidth]{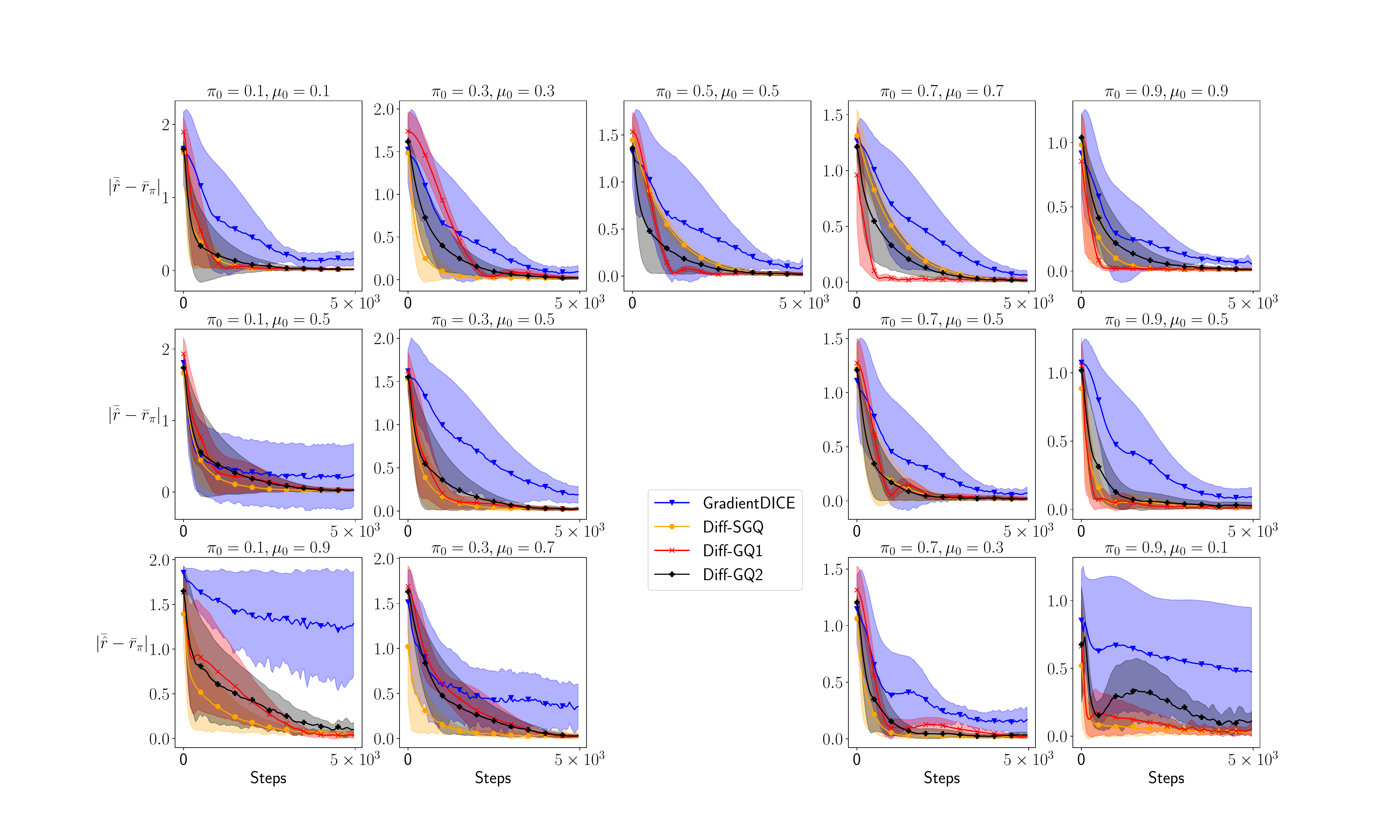}
\caption{\label{fig:boyans_chain_result}
Boyan's chain with linear function approximation. 
We vary $\pi_0$ in $\qty{0.1, 0.3, 0.5, 0.7, 0.9}$.
In the first row, we use $\mu_0 = \pi_0$;
in the second row, we use $\mu_0 = 0.5$;
in the third row, we use $\mu_0 = 1 - \pi_0$.
$\bar{\hat{r}}$ is the average $\hat{r}$ of recent 100 steps.
}
\end{figure*}

\begin{figure}[t]
\centering
\includegraphics[width=0.8\linewidth]{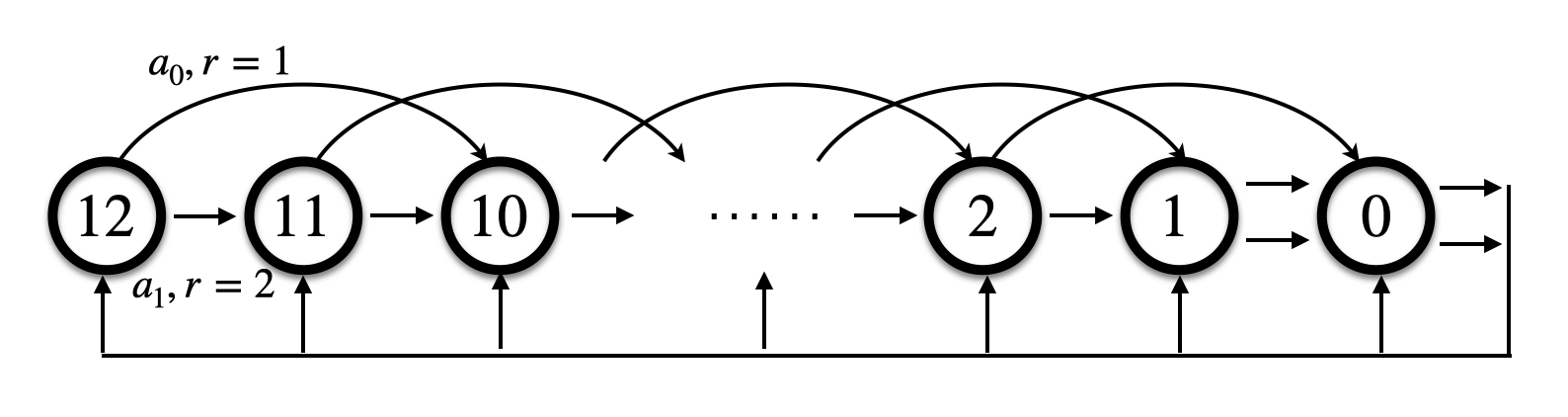}
\caption{\label{fig:boyans_chain}
A variant of Boyan's chain for policy evaluation in the average-reward setting.
There are 13 states $\qty{s_0, \dots, s_{12}}$ and two actions $\qty{a_0, a_1}$ in the chain.
For any $i \in \qty{0, \dots, 12}$, $r(s_i, a_0) = 1$ and $r(s_i, a_1) = 2$.
For any $i \geq 2$, $p(s_{i-2}|s_i, a_0) = 1$ and $p(s_{i-1}|s_i, a_1) = 1$.
At $s_1$, both actions lead to $s_0$.
At $s_0$, both actions lead to a random state in $\qty{s_0, \dots, s_{12}}$ with equal probability.
}
\end{figure}

\begin{figure*}[h]
\centering
\includegraphics[width=0.8\textwidth]{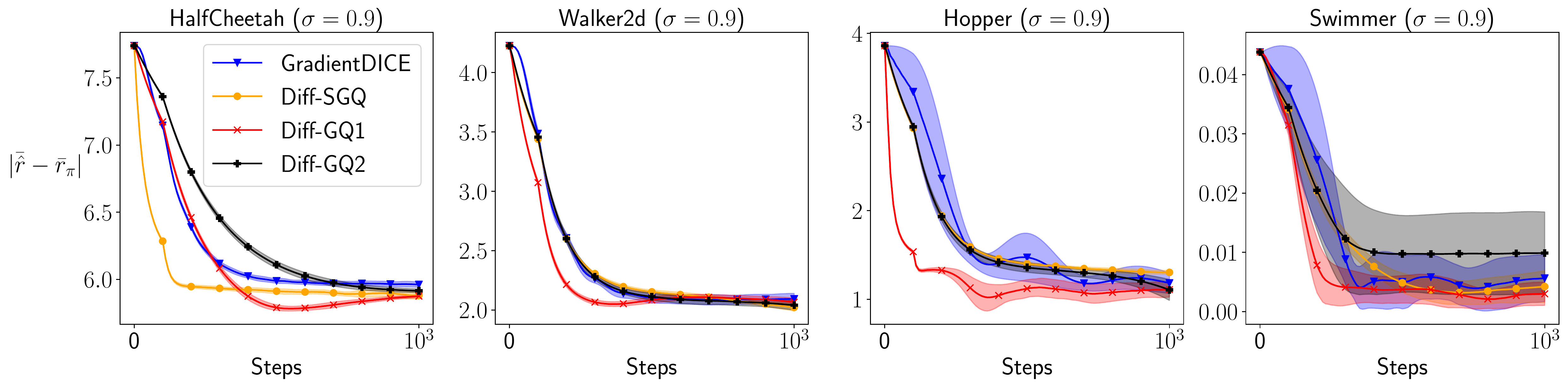}
\caption{\label{fig:mujoco_result} MuJoCo tasks with with neural network function approximation.
$\bar{\hat{r}}$ is the average $\hat{r}$ of recent 100 steps.
} 
\end{figure*}

\section{Experiments}
In light of the reproducibility challenge in RL research \citep{henderson2017deep},
we perform a grid search with 30 independent runs for hyperparameter tuning in all our experiments.
Each curve corresponds to the best hyperparameters minimizing the error of the reward rate prediction at the end of training and is averaged over 30 independent runs with the shaded region indicating one standard deviation.
To the best of our knowledge,
GradientDICE is the only density-ratio-based approach for off-policy policy evaluation in average-reward MDPs that is provably convergent with general linear function approximation and has $\mathcal{O}(K)$ computational complexity per step.
We, therefore, use GradientDICE as a baseline.
See Section~\ref{sec:gradientdice} for more details about GradientDICE.
All the implementations are publicly available.~\footnote{\url{https://github.com/ShangtongZhang/DeepRL}} 

\textbf{Linear Function Approximation.}
We benchmark Diff-SGQ, 
Diff-GQ1, Diff-GQ2, and GradientDICE in a variant of Boyan's chain \citep{boyan1999least},
which is the same as the chain used in \citet{zhang2020gradientdice}
except that we introduce a nonzero reward for each action
for the purpose of policy evaluation.
The chain is illustrated in Figure~\ref{fig:boyans_chain}.
We consider target policies of the form $\pi(a_0 | s_i) = \pi_0$ for all $s_i$,
where $\pi_0 \in [0, 1]$ is some constant.
The sampling distribution we consider has the form $d_\mu(s_i, a_0) = \frac{\mu_0}{13}$ and $d_\mu(s_i, a_0) = \frac{1 - \mu_0}{13}$
for all $s_i$,
where $\mu_0 \in [0, 1]$ is some constant.
Even if $\mu_0 = \pi_0$, 
the problem is still off-policy.
We consider linear function approximation and use the same state features as \citet{boyan1999least},
which are detailed in Section~\ref{sec:impl}.
We use an one-hot encoding for actions.
Concatenating the state feature and the one-hot action feature yields the state-action feature we use in the experiments. 

% Similar to Differential FQE,
% we can also introduce $\zeta$ in Diff-GQ1 and Diff-GQ2 so that $\nu$ and $w$ are updated at different speeds.
We use constant learning rates $\alpha$ for all compared algorithms,
which is tuned in $\qty{2^{-20}, 2^{-19}, \dots, 2^{-1}}$.
% For Differential FQE, 
% we set $\zeta = 1$.
% \sz{Maybe tune $\zeta$ as well}.
For Diff-GQ1 and Diff-GQ2,
besides tuning $\alpha$ in the same way as Diff-SGQ,
we tune $\eta$ in $\qty{0, 0.01, 0.1}$.
For GradientDICE,
besides tuning $(\alpha, \eta)$ in the same way as Diff-GQ1, 
we tune $\lambda$,
the weight for a normalizing term,
in $\qty{0, 0.1, 1, 10}$.

We run each algorithm for $5 \times 10^3$ steps.
Diff-GQ2 updates are applied every two steps as one Diff-GQ2 update requires two samples.
The results in Figure~\ref{fig:boyans_chain_result} suggest that the three differential-value-based algorithms proposed in this paper consistently outperform the density-ratio-based algorithm GradientDICE in the tested domain.

\textbf{Nonlinear Function Approximation.}
The value-based off-policy policy evaluation algorithms proposed in this paper can be easily combined with neural network function approximators.
For Diff-SGQ, 
we use a target network \citep{mnih2015human} to stabilize the training of neural networks.
For Diff-GQ1 and Diff-GQ2, 
we introduce neural network function approximators in the saddle-point objectives (i.e., $J_{1, \eta}(u, \nu)$ and $J_{2, \eta}(w, \nu)$) directly, 
similar to \citet{zhang2020gradientdice} in GradientDICE.
The details are provided Sections~\ref{sec:neural-fqe},~\ref{sec:neural-gq1}, and~\ref{sec:neural-gq2}. 

We benchmark Diff-SGQ, Diff-GQ1, Diff-GQ2, and GradientDICE in several MuJoCo domains.
To this end, 
we first train a deterministic target policy $\pi_0$ with TD3 \citep{fujimoto2018addressing}.
The behavior policy $\mu_0$ is composed by introducing Gaussian noise to $\pi_0$, i.e., $\mu_0(a | s) \doteq \mathcal{N}(\pi_0(s), \sigma^2 I)$.
The ground truth reward rate of $\pi_0$ is computed with Monte Carlo methods by running $\pi_0$ for $10^6$ steps.
We vary $\sigma$ from $\qty{0.1, 0.5, 0.9}$.
More details are provided in Section~\ref{sec:impl}.
For Differential FQE, Diff-GQ1, and Diff-GQ2,
we tune the learning rate from $\qty{0.1, 0.05, 0.01, 0.005, 0.001}$.
For GradientDICE,
we additionally tune $\lambda$ from $\qty{0.1, 1, 10}$.
The results with $\sigma = 0.9$  are reported in Figure~\ref{fig:mujoco_result},
where Diff-GQ1 consistently performs the best. 
The results with $\sigma = 0.1$ and $\sigma = 0.5$ are deferred to Section~\ref{sec:mujoco-expt},
where Diff-GQ1 consistently performs the best as well.

\section{Related Work and Discussion}

In this paper,
we addressed the policy evaluation problem with function approximation in the model-free setting.
If the model is given or learned by the agent, such a problem could be solved by, for example, classic approximate dynamic programming approaches \citep{powell2007approximate}, search algorithms \citep{russell2002artificial}, and other optimal control algorithms \citep{kirk2004optimal}.
% The model can also be learned from data,
% after which it becomes natural to address the policy evaluation problem in an on-policy manor.
For more discussion about learning a model, see, for example, \citet{sutton1990integrated,sutton2012dyna,liu2018representation,chua2018deep,wan2019planning,gottesman2019combining,yu2020mopo,kidambi2020morel}.
% While model-based approaches with a learned model could possibly be more sample-efficient and adapts faster to non-stationary environments, these approaches need to resolve
% additional challenging problems that do not need to be resolved by model-free approaches. 
% their we note several fundamental 
% Choosing an appropriate function approximator for the model, learning the model with low error,  planning efficiently even with model error.
% both of which do not exist in model-free learning.  

The on-policy average-reward policy evaluation problem was studied by \citet{tsitsiklis1999average}, which proposed and solved a Projected Bellman Equation (PBE).
The reward rate in PBE is a known quantity, which is trivial to estimate in the on-policy case.
The reward rate, however, cannot be obtained easily in the off-policy case and needs to be estimated cleverly. Such a challenge is resolved in our work by optimizing a novel objective, $\text{MSPBE}_1$, which has the reward rate estimate as a free variable to optimize.
Moreover, by proposing the other novel objective $\text{MSPBE}_2$, we showed that the reward rate or its direct estimate does not even have to appear in an objective. In fact, for the uniqueness of the solution, our algorithms did not optimize $\text{MSPBE}_1$ and MSPBE$_2$, but optimized a regularized version of these objectives, where the weight of the regularization term can be arbitrarily small.
Introducing a regularization term in MSPBE-like objectives is not new though; see, for example, \citet{mahadevan2014proximal,yu2017convergence,du2017stochastic,zhang2019provably,zhang2020gradientdice}.
One could,
of course,
apply regularization to Diff-SGQ directly,
similar to \citet{diddigi2019convergent} in the discounted off-policy linear TD.
Unfortunately, 
% the regularization in \citet{diddigi2019convergent} is used to guarantee the stability of the algorithm and thus
the weight for their regularization term has to be sufficiently large to ensure convergence.
% We conjecture that such a requirement also needs to be satisfied in Diff-SGQ.
% By contrast,
% the weight for regularization term in our objectives can be arbitrarily small.

Fenchel's duality, which we used in the derivation of our algorithms, is not new in RL research. For example, it has been applied to cope with the double sampling problem (see, e.g., \citet{liu2015finite,macua2014distributed,dai2017sbeed,liu2018block,nachum2019dualdice,nachum2019algaedice,zhang2020gendice,zhang2020gradientdice}) or to construct novel policy iteration frameworks \citep{zhang2020per}.

\section{Conclusion}
In this paper, we provided the first study of the off-policy policy evaluation problem (estimating both reward rate and differential value function) in the function approximation, average-reward setting. Such a problem encapsules the existing off-policy evaluation problem (estimating only the reward rate;
see, e.g., \citet{li2019perspective}).
To this end, we proposed two novel MSPBE objectives and derived two algorithms optimizing regularized versions of these objectives.
The proposed algorithms are the first provably convergent algorithms for estimating the differential value function and are also the first provably convergent algorithms for estimating the reward rate without estimating density ratio in this setting.
In terms of estimating the reward rate,
though our goal is not to achieve new state of the art,
our empirical results confirmed that the proposed value-based methods consistently outperform a competitive density-ratio-based method in tested domains.
We conjecture that this performance advantage results from the flexibility of value-based methods,
that is, for any $c$, $q_\pi + c \tb{1}$ is a feasible learning target.
By contrast, the density ratio $\frac{d_\pi(s, a)}{d_\mu(s, a)}$ is unique. 
Overall, 
our empirical study suggests that value-based methods deserve more attention in future research on off-policy evaluation in average-reward MDPs. 

\section*{Acknowledgments}
SZ is generously funded by the Engineering and Physical Sciences Research Council (EPSRC). This project has received funding from the European Research Council under the European Union's Horizon 2020 research and innovation programme (grant agreement number 637713). The experiments were made possible by a generous equipment grant from NVIDIA. 

% \newpage

\bibliography{ref.bib}
\bibliographystyle{icml2021}

% \end{document}
\onecolumn
\newpage
\appendix

\section{Proofs}
We first state a general result from \citet{borkar2009stochastic} which will be repeatedly used.
Consider updating $\kappa \in \R^K$ as
\begin{align}
\kappa_{k+1} \doteq \kappa_{k} + \alpha_k (G_{k+1} \kappa_k + h_{k+1} + o(1)),
\end{align}
where $G_{k+1} \in \R^{K \times K}, h_{k+1} \in \R^K$, 
and $o(1)$ denotes some bounded random or deterministic noise that converges to 0 as $k \to \infty$. 
Assuming
\begin{assumption}
\label{assu:borkar-lr}
$\qty{\alpha_k}$ is a positive deterministic nonincreasing sequence satisfying $\sum_k \alpha_k = \infty, \sum_k \alpha_k^2 < \infty$
\end{assumption}
\begin{assumption}
\label{assu:borkar-mds}
There exist $\bar{G} \in \R^{K \times K}$ and $\bar{h} \in \R^K$ such that
\begin{align}
M_{k+1} &\doteq G_{k+1}\kappa_k + h_{k+1} - \bar{G}\kappa_k - \bar{h}
\end{align}
satisfies
\begin{enumerate}
	\item $\E[M_{k+1} \mid \mathcal{F}_k] = \tb{0}$ a.s..
	\item $\E[\norm{M_{k+1}}^2 \mid \mathcal{F}_k] \leq C(1 + \norm{\kappa_k}^2)$ for some constant $C > 0$ a.s.. 
\end{enumerate}
Here
\begin{align}
\mathcal{F}_k \doteq \sigma(x_0, M_1, M_2, \dots, M_k),
\end{align}
where $\sigma(\cdot)$ denotes the $\sigma$-field.
\end{assumption}
\begin{assumption}
\label{assu:borkar-nd}
The real part of every eigenvalue of $\bar{G}$ is strictly negative.
\end{assumption}
\begin{theorem}
\label{thm:borkar}
\citep{borkar2009stochastic}
Under Assumptions~\ref{assu:borkar-lr}-~\ref{assu:borkar-nd},
almost surely,
\begin{align}
\lim_{k \to \infty} \kappa_k = -\bar{G}^{-1}\bar{h}
\end{align}
\end{theorem}
Theorem~\ref{thm:borkar} combines the third extension of Theorem 2 in Chapter 2.2 and Theorem 7 in Chapter 3 of \citet{borkar2009stochastic}.

\subsection{Proof of Theorem~\ref{thm:gq1-convergence}}
\label{sec:proof-gq1}
\begin{proof}
The proof mimics the convergence proof of GTD2 in \citet{sutton2009fast}.
We proceed by verifying Assumptions~\ref{assu:borkar-lr}-~\ref{assu:borkar-nd} thus invoking Theorem~\ref{thm:borkar}.
With $\kappa_k \doteq [\nu_k^\top, u_k^\top]^\top$,
we rewrite~\eqref{eq:gq1-update} as
\begin{align}
\kappa_{k+1} \doteq \kappa_k + \alpha_k (G_{k+1} \kappa_k + h_{k+1}),
\end{align}
where 
\begin{align}
G_{k+1} &\doteq \mqty[-y_k y_k^\top & y_k(y_k' - y_k)^\top - y_k e_1^\top \\
(y_k - y_k')y_k^\top  + e_1 y_k^\top & -\eta I_0], \\
h_{k+1} &\doteq \mqty[y_k r_k \\ \tb{0}], I_0 \doteq \mqty[0 &\tb{0}^\top \\ \tb{0} & I].
\end{align}
The asymptotic behavior of $\qty{\kappa_k}$ is governed by
\begin{align}
\bar{G} &\doteq \E[G_{k+1}] = \mqty[-C & A \\ -A^\top & - \eta I_0], \\
\bar{h} &\doteq \E[h_{k+1}] = \mqty[b \\ \tb{0} ],
\end{align} 
where 
\begin{align}
A &\doteq Y^\top D(P_\pi - I) Y - Y^\top d_\mu e_1^\top \\
&= \mqty[-1 &\tb{1}^\top D(P_\pi - I)X \\-X^\top d_\mu & X^\top D(P_\pi - I)X] \\
b &\doteq \mqty[Y^\top D r \\ \tb{0}].
\end{align}
Assumption~\ref{assu:borkar-lr} is satisfied by our requirement on $\qty{\alpha_k}$.
For Assumption~\ref{assu:borkar-mds},
we define
\begin{align}
M_{k+1} &\doteq G_{k+1}\kappa_k + h_{k+1} - \bar{G}\kappa_k - \bar{h}.
\end{align}
It is easy to see
\begin{align}
\E[M_{k+1} \mid \mathcal{F}_k] &= \E[G_{k+1}] \kappa_k + \E[h_{k+1}] - \bar{G} \kappa_k - \bar{h} = \tb{0} \\
\E[\norm{M_{k+1}}^2 \mid \mathcal{F}_k] &\leq \frac{1}{2}\E[\norm{G_{k+1} - \bar{G}}^2\norm{\kappa_k}^2 + \norm{h_{k+1} - \bar{h}}^2 | \mathcal{F}_k].
\end{align}
Because our samples are generated in an i.i.d fashion, Assumption~\ref{assu:borkar-mds} is guaranteed to hold.

To verify Assumption~\ref{assu:borkar-nd},
we first show $\det(\bar{G}) \neq 0$.
Using the rule of block matrix determinant, 
we have
\begin{align}
\det(\bar{G}) = \det(C)\det(\eta I_0 + A^\top C^{-1}A).
\end{align}
Assumption~\ref{assu:nonconstant} ensures $C$ is positive definite and
$A^\top C^{-1}A$ is positive semidefinite,
implying $\eta I_0 + A^\top C^{-1} A$ is positive semidefinite.
For any $z \neq \tb{0} \in \R^{k+1}$,
$z^\top I_0 z = 0$ if and only if $z$ has the form $\mqty[c \\ \tb{0}]$ for some $c \neq 0 \in \R$,
implying $A^\top z \neq \tb{0}$,
i.e., $z^\top A^\top C^{-1} A z \neq 0$.
So as long as $\eta > 0$, $z^\top (\eta I_0 + A^\top C^{-1} A)z \neq 0$,
implying $\eta I_0 + A^\top C^{-1} A$ is positive definite.
It follows easily that $\det(\bar{G}) \neq 0$.
Let $\lambda \in \C$ be an eigenvalue of $\bar{G}$.
$\det(\bar{G}) \neq 0$ implies $\lambda \neq 0$.
Let $z \neq \tb{0} \in \C^{2K+2}$ be the corresponding normalized eigenvector of $\lambda$, i.e., $z^H z = 1$, 
where $z^H$ is the conjugate transpose of $z$.
Let $z = \mqty[z_1 \\ z_2]$, we have
\begin{align}
\lambda &= z^H \bar{G} z =-z_1^H Cz_1 - z_2^HA^\top z_1 + z_1^H A z_2 - \eta z_2^H I_0 z_2.
\end{align}
As $(z_2^HA^\top z_1)^H = z_1^HAz_2$,
we have $\Re(- z_2^HA^\top z_1 + z_1^H A z_2) = 0$,
where $\Re(\cdot)$ denotes the real part.
So 
\begin{align}
\Re(\lambda) = -z_1^H Cz_1 - \eta z_2^H I_0 z_2 \leq 0.
\end{align}
Because $\lambda \neq 0$, we have $\Re(\lambda) < 0$.
Assumption~\ref{assu:borkar-nd} then holds.
Invoking Theorem~\ref{thm:borkar} yields 
\begin{align}
\lim_k \kappa_k = -\bar{G}^{-1} \bar{h} \qq{almost surely.}
\end{align}
Let $u^*_\eta$ be the lower half of $-\bar{G}^{-1}\bar{h}$,
we have
\begin{align}
u^*_\eta \doteq -(\eta I_0 + A^\top C^{-1} A)^{-1}A^\top C^{-1}b.
\end{align} 
From~\eqref{eq:gq1-saddlepoint}, 
we can rewrite $J_{1, \eta}(u)$ as
\begin{align}
J_{1, \eta}(w)
% \norm{Y^\top D(P_\pi - I)Yu - Y^\top d_\mu e_1^\top u + Y^\top Dr}^2_{C^{-1}} + \eta \norm{u}^2 \\
&= \norm{Au + b}_{C^{-1}}^2 + \eta u^\top I_0 u
\end{align}
It is easy to verify (e.g., using the first order optimality condition of $J_{1, \eta}(u)$) that $u^*_\eta$ is the unique minimizer of $J_{1, \eta}(u)$.

It can also be seen that if $\eta = 0$ and $A$ is invertible,
$\det(\bar G) \neq 0$ as well
and $u_{\eta = 0}^* = - A^{-1} b = u_{TD}$.
\end{proof}

\subsection{Proof of Proposition~\ref{prop:gq1-model-error}}
\label{sec:gq1-model-error}
\begin{proof}
% Note 
% \begin{align}
% J_{1, \eta}(u) = \norm{C^{-\frac{1}{2}}Au + C^{-\frac{1}{2}}b}^2 + \eta \norm{u}^2.
% \end{align}
% Assumption~\ref{assu:gq1-solvable} implies
% \begin{align}
% C^{-\frac{1}{2}}Au + C^{-\frac{1}{2}}b = 0
% \end{align}
% is solvable.
% From standard arguments of system of linear equations,
% there exists a $z \in \R^{K+1}$ such that
% \begin{align}
% u^* = -(C^{-\frac{1}{2}}A)^
% \end{align}

% \begin{align}
% C^{-\frac{1}{2}}A (C^{-\frac{1}{2}}A)^\dagger C^{-\frac{1}{2}}b = C^{-\frac{1}{2}}b,
% \end{align}
% using which it is easy to verify 
% \begin{align}
% C^{-\frac{1}{2}}Au^*_0 + C^{-\frac{1}{2}}b = 0,
% \end{align}
% implying
$u^*$ is a TD fixed point implies
% $Au^* + b = 0$ implies
\begin{align}
    \E[\delta_k(u^*) y_k] = \tb{0},
\end{align}
which implies
\begin{align}
Y^\top D(P_\pi - I)Yu^* - Y^\top d_\mu e_1^\top u^* + Y^\top Dr = \tb{0},
\end{align}
expanding which yields
\begin{align}
\hat{r}^* - d_\mu^\top (r + P_\pi X w^* - X w^*) &= \tb{0}, \\
X^\top D (r - \hat{r}^* \tb{1} + P_\pi X w^* - X w^*) &= \tb{0}.
\end{align}
So we have
\begin{align}
\norm{X^\top D (r - \hat{r}^* \tb{1} + P_\pi X w^* - X w^*)}_{C^{-1}}^2 = \tb{0},
\end{align}
implying
\begin{align}
\norm{\Pi_X (r - \hat{r}^* \tb{1} + P_\pi X w^* - X w^*)}_{D}^2 = \tb{0}.
\end{align}
% In the rest of this proof, we use $\Pi$ as a shorthand of $\Pi_X$. 
Using the Schur complement,
Assumption~\ref{assu:mupi} implies (see \citet{kolter2011fixed} for more details)
\begin{align}
\norm{\Pi_X P_\pi Xw}_D \leq \xi \norm{Xw}_D
\end{align}
holds for any $w \in \R^K$.
We have
\begin{align}
&\norm{Xw^* - q_\pi^c}_D \\
\leq& \norm{Xw^* - \Pi_X q_\pi^c}_D + \norm{\Pi_X q_\pi^c - q_\pi^c}_D \\
=&\norm{\Pi_X(r + P_\pi X w^* - \hat{r}^* \tb{1}) - \Pi_X(r + P_\pi q_\pi^c - r_\pi \tb{1})}_D + \norm{\Pi_X q_\pi^c - q_\pi^c}_D \\
\leq& \norm{\Pi_X P_\pi X w^* - \Pi_X P_\pi q_\pi^c}_D + \norm{\Pi_X (\hat{r}^* \tb{1} - r_\pi \tb{1})}_D + \norm{\Pi_X q_\pi^c - q_\pi^c}_D  \\
=& \norm{\Pi_X P_\pi X w^* - \Pi_X P_\pi q_\pi^c}_D + \norm{X (X^\top D X)^{-1}(X^\top D \tb{1})(\hat{r}^* - r_\pi )}_D + \norm{\Pi_X q_\pi^c - q_\pi^c}_D \\
=& \norm{\Pi_X P_\pi X w^* - \Pi_X P_\pi q_\pi^c}_D + \norm{\Pi_X q_\pi^c - q_\pi^c}_D  \qq{(Using $X^\top d_\mu = \tb{0}$)}\\
\leq& \norm{\Pi_X P_\pi X w^* - \Pi_X P_\pi \Pi_X q_\pi^c}_D + \norm{\Pi_X P_\pi \Pi_X q_\pi^c - \Pi_X P_\pi q_\pi^c}_D + \norm{\Pi_X q_\pi^c - q_\pi^c}_D \\
\leq& \xi \norm{Xw^* - \Pi_X q_\pi^c}_D + \norm{P_\pi}_D \norm{\Pi_X q_\pi^c - q_\pi^c}_D  + \norm{\Pi_X q_\pi^c - q_\pi^c}_D \\
=& \xi \norm{Xw^* - q_\pi^c}_D + (\norm{P_\pi}_D + 1) \norm{\Pi_X q_\pi^c - q_\pi^c}_D.
\end{align}
% The third equality holds because $X^\top d_\mu = 0$. 
From the above derivation we have
\begin{align}
\norm{Xw^* - q_\pi^c}_D \leq \frac{\norm{P_\pi}_D + 1}{1 - \xi} \norm{\Pi_X q_\pi^c - q_\pi^c}_D.
\end{align}
Take the infimum
\begin{align}
    \inf_{c \in \R} \norm{Xw^* - q_\pi^c}_D \leq \inf_{c \in \R} \frac{\norm{P_\pi}_D + 1}{1 - \xi} \norm{\Pi_X q_\pi^c - q_\pi^c}_D.
\end{align}
For the reward rate at the fixed point, we have, for all $c \in \R$,
\begin{align}
|r_\pi - \hat{r}^*| &= |d_\mu^\top (P_\pi - I)(Xw^* - q_\pi^c)| \\
&= |d_\mu^\top (P_\pi - I)D^{-\frac{1}{2}}D^{\frac{1}{2}}(Xw^* - q_\pi^c)| \\
&\leq \norm{d_\mu^\top (P_\pi - I)}_{D^{-1}} \norm{Xw^* - q_\pi^c}_D,
% &= \norm{\tb{1} d_\mu^\top (P_\pi - I)(Xw^* - q_\pi^c)}_D\\
% &\leq \norm{\tb{1} d_\mu^\top (P_\pi - I)}_D \norm{Xw^* - q_\pi^c}_D\\
% &= \norm{d_\mu^\top (P_\pi - I)} \norm{Xw^* - q_\pi^c}_D\\
% &\leq \frac{\norm{d_\mu^\top (P_\pi - I)}(\norm{P_\pi}_D + 1)}{1 - \xi} \inf_{c \in \R} \norm{\Pi_X q_\pi^c - q_\pi^c}_D.
\end{align}
where the inequality is due to the Cauchy-Schwarz inequality.

\end{proof}

\subsection{Projected Diff-GQ1}
\label{sec:pgq1}
The Projected Diff-GQ1 optimizes the MSPBE$_1$ objective: 
\begin{align}
\text{MSPBE}_1(u) = \max_{\nu \in \R^{K+1}} J_{1, \eta=0}(u, \nu),
\end{align}
where 
\begin{align}
J_{1, \eta=0}(u, \nu) \doteq 2 \nu^\top Y^\top D \bar\delta(u) - \nu^\top C \nu.
\end{align}
Similar to the Revised GTD Algorithms in \citet{liu2015finite}, the Projected Diff-GQ1 update $u_k$ and $\nu_k$ as
\begin{align}
\nu_{k+1} &\doteq \Pi_{\Theta_1} \Big(\nu_k + \alpha \big(-y_k y_k^\top \nu_k + y_k (y_k' - y_k)^\top u_k - y_k e_1^\top u_k + y_k R_k \big) \Big), \\
u_{k+1} &\doteq \Pi_{\Theta_2} \Big(u_k + \alpha \big( (y_k - y_k')y_k^\top \nu_k + e_1y_k^\top \nu_k \big) \Big),
\end{align}
where $\alpha$ is a constant learning rate, 
$\Theta_i$ is a compact subset in $\R^{K+1}$ and $\Pi_{\Theta_i}$ is projection into $\Theta_i$ w.r.t. $\norm{\cdot}$.
% Note $J'_{1, \eta}(u, \nu)$ is strictly convex in $u$ and strictly concave in $\nu$,
% indicating it adopts a unique saddle point,
If $A$ is nonsingular,
$J_{1, \eta=0}(u, \nu)$ has a unique saddle point,
which we refer to as $(u^*, \nu^*)$.
It is easy to see $u^*$ is the unique minimizer of $J_{1, \eta=0}(u)$.
% Let $\bar{u}_k \doteq \frac{1}{k} \sum_{i=1}^k u_k$,
We have
\begin{proposition}
If Assumptions~\ref{assu:chain},~\ref{assu: positive dmu},~\ref{assu:linearly independent},~\ref{assu:nonconstant},~\ref{assu:mupi}, \&~\ref{assu: 0 mean feature vector} hold, $\nu^* \in \Theta_1, u^* \in \Theta_2$, and $A$ is nonsingular, with properly tuned $\alpha$, for any $\delta \in (0, 1)$, at least with probability $1 - \delta$,
the iterate $\qty{\hat{u}_k = [\hat r_k, w_k^\top]^\top}$ generated by Projected Diff-GQ1 satisfies
\begin{align}
\left(\bar{\hat{r}}_k - r_\pi \right)^2 = \mathcal{O}\left(\frac{C_1 \delta - C_2 \delta \ln \delta}{\sqrt{k}}\right) + \mathcal{O}\left(\inf_{c \in \R} \norm{\Pi_X q_\pi^c - q_\pi^c}_D^2\right), \\
\inf_{c \in \R}\norm{X \bar w_k - q_\pi^c }^2 = \mathcal{O}\left(\frac{C_1 \delta - C_2 \delta \ln \delta}{\sqrt{k}}\right) + \mathcal{O}\left(\inf_{c \in \R} \norm{\Pi_X q_\pi^c - q_\pi^c}_D^2\right),
\end{align}
where $q_\pi^c \doteq q_\pi + c \tb{1}$, $\bar{\hat r}_k \doteq (1 / k) \sum_{i=1}^{k} \hat{r}_i, \bar{w}_k \doteq (1 / k) \sum_{i=1}^{k} w_i$, $C_1, C_2 > 0$ are constants
\end{proposition}
\begin{proof}
We first state a lemma.
\begin{lemma}
\label{lem:boliu}
With at least probability $1 - \delta$,
\begin{align}
\norm{\frac{1}{k}\sum_{i=1}^k u_i - u^*}^2 \leq K_0 \sqrt{\frac{5}{k}} \left(8 + 2\ln \frac{2}{\delta}\right) \delta = \mathcal{O}\left(\frac{C_1 \delta - C_2 \delta \ln \delta}{\sqrt{k}}\right),
\end{align}
where $K_0, C_1, C_2 > 0$ are constants.
\end{lemma}
\begin{proof}
The proof is the same as the finite sample analysis of GTD2 in \citet{liu2015finite} up to change of notations (see Proposition 3, the proof of Theorem 1, the proof of Proposition 5 in \citet{liu2015finite}).
We, therefore, omit the proof to avoid verbatim repetition.
\end{proof}
Note that
\begin{align}
\left(\frac{1}{k} \sum_{i =1}^k \hat{r}_i - r_\pi\right)^2 &\leq 2\left(\frac{1}{k} \sum_{i =1}^k \hat{r}_i - r^*\right)^2 + 2 (r^* - r_\pi)^2 \leq 2 \norm{\frac{1}{k}\sum_{i=1}^k u_i - u^*}^2 + 2 (r^* - r_\pi)^2,
\end{align}
and for any $c \in \R$,
\begin{align}
\norm{\frac{1}{k} \sum_{i =1}^k X w_i - q_\pi^c}^2 &\leq 2\norm{\frac{1}{k} \sum_{i =1}^k X w_i - Xw^*}^2 + 2 \norm{Xw^* - q_\pi^c}^2 \leq 2 \norm{X}^2 \norm{\frac{1}{k}\sum_{i=1}^k u_i - u^*}^2 + 2 \norm{Xw^* - q_\pi^c}^2.
\end{align}
Invoking Proposition~\ref{prop:gq1-model-error} and Lemma~\ref{lem:boliu} to bound the RHS of the above equations completes the proof.
\end{proof}

% \subsection{Proof of Lemma~\ref{lem:gq1-ridge-error}}
% \begin{proof}
% Using SVD, 
% we have
% \begin{align}
% C^{-\frac{1}{2}}A = U^\top \Sigma V,
% \end{align}
% where $U$ and $V$ are orthonormal and $\Sigma$ is a diagonal matrix.
% So
% \begin{align}
% (\eta I_0 + A^\top C^{-1}A)^{-1} A^\top C^{-1}b = 
% \end{align}
% \begin{align}
% \norm{u^*_\eta - u^*_0} &= \norm{V^{\top} \big( (\Sigma^2 + \eta I)^{-1}\Sigma - \Sigma^\dagger \big) U C^{-\frac{1}{2}} b} \\
% &\leq \frac{\norm{C^{-\frac{1}{2}}b} \eta}{\sigma_{\min}^3(C^{-\frac{1}{2}}A)}.
% \end{align}
% \end{proof}

\subsection{Proof of Theorem~\ref{thm:gq2-convergence}}
\label{sec:proof-gq2}
\begin{proof}
With $\kappa_k \doteq [\nu_{k+1}^\top, w_{k+1}^\top]^\top$, 
we rewrite~\eqref{eq:gq2-update-uw} as
\begin{align}
\kappa_{k+1} = \kappa_k + \alpha_k (G_{k+1} \kappa_k + h_{k+1}),
\end{align}
where
\begin{align}
G_{k+1} &\doteq \mqty[-x_{k, 1} x_{k, 1}^\top & x_{k, 1}(x_{k, 1}'^\top - x_{k, 1}^\top) -  x_{k, 1}(x_{k, 2}'^\top - x_{k, 2}^\top) \\ 
-(x_{k, 1} - x_{k, 1}')x_{k, 1}^\top + (x_{k, 2} - x_{k, 2}')x_{k, 1}^\top & -\eta I], \\
h_{k+1} &\doteq \mqty[r_{k, 1} x_{k, 1} - r_{k, 2} x_{k, 1}\\ \tb{0}]. 
\end{align}
The asymptotic behavior of $\qty{\kappa_k}$ is governed by
\begin{align}
\bar{G} &\doteq \E[G_{k+1}] = \mqty[-C_2 & {A_2} \\ -{A_2}^\top & -\eta I] \\
\bar{h} &\doteq \E[h_{k+1}] = \mqty[{b_2} \\ \tb{0}],
\end{align} 
where ${A_2} \doteq X^\top (D - d_\mu d_\mu^\top)(P_\pi - I)X , b_2 \doteq X^\top (D - d_\mu d_\mu^\top) r$.
Similar to the proof of Theorem~\ref{thm:gq1-convergence} in Section~\ref{sec:proof-gq1},
up to change of notations,
we can get
\begin{align}
\lim_{k \to \infty} w_k = w^*_\eta,
\end{align}
where
\begin{align}
w^*_\eta \doteq -(\eta I + A_2^\top C_2^{-1} A_2)^{-1}A_2^\top C_2^{-1}b_2
\end{align} 
is the unique minimizer of $J_{2, \eta}(w)$.
We then rewrite~\eqref{eq:gq2-update-r} as
\begin{align}
\hat{r}_{k+1} \doteq \hat{r}_k + \beta_k \Big(\frac{1}{2}\sum_{i=1}^2 (r_{k, i} + x_{k, i}'^\top w^*_\eta - x_{k, i}^\top w^*_\eta) - \hat{r}_k + o(1)\Big).
\end{align}
Similar to the convergence proof of $\qty{\kappa_k}$,
we can obtain 
\begin{align}
\lim_{k \to \infty} \hat{r}_k = d_\mu^\top (r + P_\pi X w^*_\eta - X w^*_\eta).
\end{align}
% Rewrite $w_\eta^*$ as
% \begin{align}
%     w_\eta^* = - (\eta I + B^\top B)^{-1} B^\top C_2^{-\frac{1}{2}} b_2,
% \end{align}
% where $B \doteq C^{-\frac{1}{2}} A$. Therefore we have
% \begin{align}
%     w_{\eta \downarrow 0}^* & = \lim_{\eta \downarrow 0} - (\eta I + B^\top B)^{-1} B^\top C_2^{-\frac{1}{2}} b_2\\
%     & = -B^\dagger C_2^{-\frac{1}{2}} b_2 \\
%     & = -\left(C_2^{-\frac{1}{2}} A_2 \right)^\dagger C_2^{-\frac{1}{2}} b_2
% \end{align}

Assumption~\ref{assu: existence of TD fixed points} implies
there exists $w$ such that,
\begin{align}
A_2 w + b_2 = \tb{0},
\end{align}
or equivalently,
\begin{align}
C_2^{-\frac{1}{2}}A_2 w + C_2^{-\frac{1}{2}} b_2 = \tb{0}
\end{align}
has unique or infinite manly solutions.
From standard results of system of linear equations,
this is equivalent to
\begin{align}
C_2^{-\frac{1}{2}}A_2 (C_2^{-\frac{1}{2}}A_2)^\dagger C_2^{-\frac{1}{2}} b_2 = C_2^{-\frac{1}{2}} b_2,
\end{align}
where $(\cdot)^\dagger$ denotes the Moore-Penrose pseudoinverse,
which always exists for any matrix.
By the property of the Moore-Penrose pseudoinverse,
it is easy to see
\begin{align}
w^*_0 \doteq \lim_{\eta \downarrow 0} w^*_\eta = -(C_2^{-\frac{1}{2}}A_2)^\dagger C_2^{-\frac{1}{2}}b_2.
\end{align}
Consequently,
we have
\begin{align}
C_2^{-\frac{1}{2}} (A_2 w^*_0 + b_2) = -C_2^{-\frac{1}{2}} A_2 (C_2^{-\frac{1}{2}}A_2)^\dagger C_2^{-\frac{1}{2}}b_2 + C_2^{-\frac{1}{2}} b_2 = \tb{0},
\end{align}
implying
\begin{align}
A_2 w^*_0 + b_2 = \tb{0}.
\end{align}

It can also be seen that if $A$ is invertible, $w_{0}^* = w_{\text{TD}}$ and $d_\mu^\top (r + P_\pi X w^*_{0} - X w^*_{0}) = \hat r_{\text{TD}}$.
Applying SVD to $C_2^{-\frac{1}{2}}A_2$ and using $\sigma$ to denote its minimum nonzero singular value,
it is easy to see
\begin{align}
\norm{w^*_\eta - w^*_0} \leq \frac{\eta}{\sigma^3} \norm{C_2^{-\frac{1}{2}} b_2}.
\end{align}
% expanding which and plug in the definition of $\hat{r}^*$ yields
% \begin{align}
% \hat{r}^* - d_\mu^\top (r + P_\pi X w^* - X w^*) &= 0, \\
% X^\top D (r - \hat{r}^* \tb{1} + P_\pi X w^* - X w^*) &= 0.
% \end{align}
% The rest is the same as the proof of Proposition~\ref{prop:gq1-model-error}.
\end{proof}

\subsection{Projected Diff-GQ2}
\label{sec:pgq2}
The Projected Diff-GQ2 objective is 
\begin{align}
J_{2, \eta = 0}(w) = \norm{A_2 w + b_2}_{C_2^{-1}}^2 = \max_{\nu \in \R^{K}} J_{2, \eta=0}(w, \nu),
\end{align}
where 
\begin{align}
J_{2, \eta=0}(w, \nu) \doteq 2 \nu^\top X^\top D (\bar{r}_w - d_\mu^\top \bar{r}_w \tb{1}) - \nu^\top C_2 \nu.
\end{align}
The Projected Diff-GQ2 update $w_k$, $\nu_k$ and $\hat{r}_k$ as
\begin{align}
\nu_{k+1} &\doteq \Pi_{\Theta_1} \Big(\nu_k + \alpha \big(-x_{k, 1} x_{k, 1}^\top \nu_k + x_{k, 1}(x_{k, 1}'^\top - x_{k, 1}^\top)w_k -  x_{k, 1}(x_{k, 2}'^\top - x_{k, 2}^\top)w_k + R_{k, 1} x_{k, 1} - R_{k, 2} x_{k, 1} \big) \Big), \\
w_{k+1} &\doteq \Pi_{\Theta_2} \Big(w_k + \alpha \big(-(x_{k, 1} - x_{k, 1}')x_{k, 1}^\top \nu_k + (x_{k, 2} - x_{k, 2}')x_{k, 1}^\top \nu_k \Big), \\
\bar{w}_{k+1} &\doteq \frac{k\bar{w}_k + w_{k+1}}{k+1} \\
\hat{r}_{k+1} &\doteq \hat{r}_k + \beta \Big(\frac{1}{2}\sum_{i=1}^2 (r_{k, i} + x_{k, i}'^\top \bar{w}_k - x_{k, i}^\top \bar{w}_k) - \hat{r}_k \Big),
\end{align}
where $\alpha$ and $\beta$ are constant learning rates, 
$\Theta_i$ is a compact subset in $\R^{K}$ and $\Pi_{\Theta_i}$ is projection into $\Theta_i$ w.r.t. $\norm{\cdot}$.
If $A_2$ is nonsingular,
$J_{2, \eta=0}(w, \nu)$ has a unique saddle point,
which we refer to as $(w^*, \nu^*)$.
It is easy to see $w^*$ is the unique minimizer of $J_{2, \eta=0}(w)$.
We have
\begin{proposition}
If Assumptions~\ref{assu:chain},~\ref{assu: positive dmu},~\ref{assu:linearly independent},~\ref{assu:mupi}, \&~\ref{assu: 0 mean feature vector} hold, $\nu^* \in \Theta_1, w^* \in \Theta_2$, and $A_2$ is nonsingular, then with properly tuned $\alpha$ and $\beta$, for any $\delta \in (0, 1)$, at least with probability $1 - \delta$,
the iterate $\qty{w_k}, \qty{\hat{r}_k}$ generated by Projected Diff-GQ2 satisfies
\begin{align}
\inf_{c \in \R}\norm{X \bar{w}_k - q_\pi^c }^2 &= \mathcal{O}\left(\frac{C_1 \delta - C_2 \delta \ln \delta}{\sqrt{k}}\right) + \mathcal{O}\left(\inf_{c \in \R} \norm{\Pi_X q_\pi^c - q_\pi^c}_D^2\right), \\
\frac{1}{k} \sum_{i = 1}^k \E\left[\left(\hat{r}_i - r_\pi\right)^2\right] &= \mathcal{O}\left(\frac{C_1 \delta - C_2 \delta \ln \delta}{\sqrt{k}}\right) + \mathcal{O}\left(\inf_{c \in \R} \norm{\Pi_X q_\pi^c - q_\pi^c}_D^2\right) + \mathcal{O}\left(1\right), 
\end{align}
where $C_1, C_2 > 0$ are constants, $\bar{w}_k \doteq (1 / k) \sum_{i=1}^{k} w_i$, and the term $\mathcal{O}\left(1\right)$ depends on the variance of $x(S_{k,i}, A_{k, i})$ and $x(S_{k,i}', A_{k, i}')$.
\end{proposition}
\begin{proof}
We first state a lemma.
\begin{lemma}
\label{lem:boliu2}
With at least probability $1 - \delta$,
\begin{align}
\norm{\bar{w}_k - w^*}^2 \leq K_0 \sqrt{\frac{5}{k}} (8 + 2\ln \frac{2}{\delta}) \delta,
\end{align}
where $K_0 > 0$ is a constant.
\end{lemma}
\begin{proof}
The proof is the same as the proof of Lemma~\ref{lem:boliu}. 
We, therefore, omit the proof to avoid verbatim repetition.
\end{proof}
We have
\begin{align}
\norm{X\bar{w}_k - q_\pi^c}^2 \leq 2 \norm{X}^2 \norm{\bar{w}_k - w^*}^2 + 2 \norm{Xw^* - q_\pi^c}^2.
\end{align}
Taking infimum both sides and invoking Lemma~\ref{lem:boliu2} and Proposition~\ref{prop:gq1-model-error} to bound RHS completes the proof of the first half.
Let $f(\hat{r}) \doteq \frac{1}{2} (\hat{r}^* - \hat{r})^2, g_k \doteq \frac{1}{2}\sum_{i=1}^2 (r_{k, i} + x_{k, i}'^\top \bar{w}_k - x_{k, i}^\top \bar{w}_k), g_k^* \doteq \frac{1}{2}\sum_{i=1}^2 (r_{k, i} + x_{k, i}'^\top w^* - x_{k, i}^\top w^*),$
we rewrite the $\hat{r}_k$ update as
\begin{align}
\hat{r}_{k+1} &= \hat{r}_k - \beta \xi_k,
% &=\hat{r}_k + \beta (g_k^* - \hat{r}_k + g_k - g_k^*),
% &=\hat{r}_k - \beta (\nabla f(\hat{r}) - (g_k^* - \hat{r}_k) + g_k - g_k^*)
\end{align}
where
\begin{align}
\xi_k \doteq -(g_k^* - \hat{r}_k) - (g_k - g_k^*),
\end{align}
in other words, $\hat{r}_k$ is updated following a noisy stochastic gradient $\xi_k$.
Let $\E_k$ denote the expectation w.r.t. $d_{\mu\pi}$ for $S_{k, i}, A_{k, i}, S_{k, i}', A_{k, i}'$.
As $\bar{w}_k$ does not depend on the samples at the $k$-th iteration,
we have
\begin{align}
\norm{\E_k[\nabla f(\hat{r}_k) - \xi_k | \hat{r}_k]} &= \norm{\E_k[g_k - g_k^* | \hat{r}_k]} \leq \norm{d_\mu^\top (P_\pi - I)X} \norm{\bar{w}_k - w^*}\\
\E_k[\norm{\nabla f(\hat{r}_k) - \xi_k}^2 | \hat{r}_k] &\leq 2\E_k[ \norm{\nabla f(\hat{r}_k) + (g_k^* - \hat{r}_k)}^2 + \norm{g_k - g_k^*}^2 | \hat{r}_k] \\
&\leq 2 \E_k[ \norm{\nabla f(\hat{r}_k) + (g_k^* - \hat{r}_k)}^2 | \hat{r}_k] + 2 K_1 \norm{\bar{w}_k - w^*}^2 \\
&\leq 2 (K_2 + K_1 \norm{\bar{w}_k - w^*}^2),
\end{align}
where $K_1$ and $K_2$ are some constants and $K_2$ depends on the variance of $x(S_{k,i}, A_{k, i})$ and $x(S_{k, i}', A_{k, i}')$.
Using Lemma~\ref{lem:boliu2} to bound $\norm{\bar{w}_k - w^*}$ and invoking Theorem 4 in the appendix of \citet{liu2019off} yields
\begin{align}
\frac{1}{k} \sum_{i=1}^k \E[\norm{\nabla f(\hat{r}_i)}^2] \leq \frac{2}{k} (f(\hat{r}_0) - f(\hat r^*)) + 2K_2 + \frac{2}{k} K_1 K_0(8 + 2\ln \frac{2}{\delta}) \delta  \sum_{i=1}^k \sqrt{\frac{5}{i}},
\end{align}
in other words,
\begin{align}
\frac{1}{k} \sum_{i = 1}^k \E\left[\left(\hat{r}_i - \hat{r}^*\right)^2\right] = \mathcal{O}\left(\frac{(C_1 \delta - C_2 \delta \ln \delta)}{\sqrt{k}}\right) + \mathcal{O}(1),
\end{align}
combining which and Proposition~\ref{prop:gq1-model-error} yields
\begin{align}
\frac{1}{k} \sum_{i = 1}^k \E\left[\left(\hat{r}_i - r_\pi\right)^2\right] &= \mathcal{O}\left(\frac{C_1 \delta - C_2 \delta \ln \delta}{\sqrt{k}}\right) + \mathcal{O}\left(\inf_{c \in \R} \norm{\Pi_X q_\pi^c - q_\pi^c}_D^2\right) + \mathcal{O}\left(1\right), 
\end{align}
which completes the proof.
\end{proof}

\section{Algorithm Details}
\subsection{GradientDICE with Linear and Nonlinear Function Approximation}
\label{sec:gradientdice}
Let $d_\pi(s, a)$ be the stationary state action distribution under the target policy $\pi$ (assuming it exists),
GradientDICE aims to learn the density ratio $\frac{d_\pi(s, a)}{d_\mu(s, a)}$.
Let $\tau: \mathcal{S} \times \mathcal{A} \to \R$, 
parameterized by $\theta_{\tau} \in \R^{K_1}$,
be the function to approximate the density ratio,
GradientDICE considers the following problem to optimize $\theta_{\tau}$:
\begin{align}
&\min_{\theta_\tau \in \R^{K_1}} \max_{\theta_\nu \in \R^{K_2}, u \in \R} \E[L_k], \qq{where}\\
&L_k \doteq \tau(S_k, A_k) \nu(S_k', A_k') - \tau(S_k, A_k) \nu(S_k, A_k) - \frac{1}{2} \nu(S_k, A_k)^2 +\lambda (u\tau(S_k, A_k) - u - \frac{u^2}{2}) + \frac{\eta}{2} \norm{\theta_\tau}^2.
\end{align}
Here $\nu: \mathcal{S} \times \mathcal{A} \to \R$, 
parameterized by $\theta_\nu$,
is an auxiliary function and $u$ is an auxiliary variable.
GradientDICE uses primal-dual algorithms to optimize $\theta_\tau, \theta_\nu, u$. 
Let $\alpha$ be a learning rate,
GradientDICE updates are
\begin{align}
\theta_{\tau, k+1} &\doteq \theta_{\tau, k} - \alpha \nabla_{\theta_\tau} L_k, \\
\theta_{\nu, k+1} &\doteq \theta_{\nu, k} + \alpha \nabla_{\theta_\nu} L_k, \\
u_{k+1} &\doteq u_k + \alpha \nabla_u L_k.
\end{align}
We could then use $\frac{1}{N}\sum_{k=1}^{N} \tau(S_k, A_k) R_k$ as an estimate of the reward rate,
which is, however, computationally expensive.
To obtain the average-reward estimate in an efficiently way,
we additionally maintain a scalar estimate $\hat{r}$ directly,
which is updated as 
\begin{align}
\hat{r}_{k+1} \doteq \hat{r}_k + \alpha (\tau(S_k, A_k)R_k - \hat{r}_k).
\end{align} 
\subsection{Diff-SGQ with Nonlinear Function Approximation}
\label{sec:neural-fqe}
Let $q_\theta: \mathcal{S} \times \mathcal{A} \to \R$
be the function to estimate the differential action-value function
parameterized by $\theta \in \R^{K_1}$
and $\hat{r} \in \R$ be the scalar estimate for the average-reward,
Diff-SGQ updates $\theta$ and $\hat{r}$ as
\begin{align}
\theta_{k+1} &\doteq \theta_k + \alpha (R_k - \bar{\hat{r}} + q_{\bar{\theta}}(S_k', A_k') - q_\theta(S_k, A_k)) \nabla_\theta q_\theta(S_k, A_k), \\
\hat{r}_{k+1} &\doteq \hat{r}_k + \alpha (R_k + q_{\bar{\theta}}(S_k', A_k') - q_{\bar{\theta}}(S_k, A_k) - \hat{r}_k),
\end{align}
where $\bar{\theta}$ and $\bar{\hat{r}}$ are parameters of the target network,
which are synchronized with $\theta$ and $\hat{r}$ periodically.

\subsection{Diff-GQ1 with Nonlinear Function Approximation}
\label{sec:neural-gq1}
Let $q \in \R^{\nsa}, \hat{r} \in \R$ be our estimates for the differential action-value function and the average-reward,  
we have
\begin{align}
&\norm{r - \hat{r}\tb{1} + P_\pi q - q}^2_D \\
=& \E[\big(R_k - \hat{r} + q(S_k', A_k') - q(S_k, A_k)\big)^2] \\
=& \E[ \max_{\tau \in \R} 2\big(R_k - \hat{r} + q(S_k', A_k') - q(S_k, A_k)\big)\tau - \tau^2] \\
=& \max_{\tau \in \R^\nsa} \E[ 2\big(R_k - \hat{r} + q(S_k', A_k') - q(S_k, A_k)\big)\tau(S_k, A_k) - \tau(S_k, A_k)^2].
\end{align}
When using function approximation,
we assume $q: \mathcal{S} \times \mathcal{A} \to \R$ is parameterized by $\theta \in \R^{K_1}$ and consider the following problem:
\begin{align}
&\min_{\theta \in \R^{K_1}, \hat{r} \in \R} \, \max_{\theta_\tau \in \R^{K_2}} \E[L_k], \qq{where} \\
&L_k \doteq 2\big(R_k - \hat{r} + q(S_k', A_k') - q(S_k, A_k)\big)\tau(S_k, A_k) - \tau(S_k, A_k)^2
\end{align}
Here the auxiliary function $\tau: \mathcal{S} \times \mathcal{A} \to \R$ is parameterized by $\theta_\tau \in \R^{K_2}$.
Diff-GQ1 updates are then
\begin{align}
\theta_{k+1} &\doteq \theta_k - \alpha \nabla_\theta L_k, \\
\hat{r}_{k+1} &\doteq \hat{r}_k - \alpha \nabla_{\hat{r}} L_k, \\
\theta_{\tau, k+1} &\doteq \theta_{\tau, k} + \alpha \nabla_{\theta_\tau} L_k.
\end{align}
If both $q$ and $\tau$ are linear,
the above updates are the same as~\eqref{eq:gq1-update} with $\eta = 0$.

\subsection{Diff-GQ2 with Nonlinear Function Approximation}
\label{sec:neural-gq2}
Let $q \in \R^{\nsa}$ be our estimates for the differential action-value function,  
we have
\begin{align}
&\norm{r - d_\mu^\top (r + P_\pi q - q)\tb{1} + P_\pi q - q}^2_D \\
=& \E[\big(R_k - d_\mu^\top (r + P_\pi q - q) + q(S_k', A_k') - q(S_k, A_k)\big)^2] \\
=& \E[ \max_{\tau \in \R} 2\big(R_k - d_\mu^\top (r + P_\pi q - q) + q(S_k', A_k') - q(S_k, A_k)\big)\tau - \tau^2] \\
=& \max_{\tau \in \R^\nsa} \E[ 2\big(R_k - d_\mu^\top (r + P_\pi q - q) + q(S_k', A_k') - q(S_k, A_k)\big)\tau(S_k, A_k) - \tau(S_k, A_k)^2].
\end{align}
When using function approximation,
we assume $q: \mathcal{S} \times \mathcal{A} \to \R$ is parameterized by $\theta \in \R^{K_1}$ and consider the following problem:
\begin{align}
&\min_{\theta \in \R^{K_1}} \, \max_{\theta_\tau \in \R^{K_2}} \E[L_k], \qq{where} \\
&L_k \doteq 2\Bigg(R_{k, 1} - \Big(R_{k, 2} + q(S_{k, 2}', A_{k, 2}') - q(S_{k, 2}, A_{k, 2}) \Big) + q(S_{k, 1}', A_{k, 1}') - q(S_{k, 1}, A_{k, 1})\Bigg)\tau(S_{k, 1}, A_{k, 1}) \\
&\quad- \tau(S_{k, 1}, A_{k, 1})^2.
\end{align}
Here the auxiliary function $\tau: \mathcal{S} \times \mathcal{A} \to \R$ is parameterized by $\theta_\tau \in \R^{K_2}$.
Diff-GQ2 updates are then
\begin{align}
\theta_{k+1} &\doteq \theta_k - \alpha \nabla_\theta L_k, \\
\theta_{\tau, k+1} &\doteq \theta_{\tau, k} + \alpha \nabla_{\theta_\tau} L_k ,\\
\hat{r}_{k+1} &\doteq \hat{r}_k + \alpha \Big( \frac{1}{2}\sum_{i=1}^2 \big( R_{k, i} + q(S_{k, i}', A_{k, i}') - q(S_{k, i}, A_{k, i}) \big) - \hat{r}_{k} \Big),
\end{align}
where $\hat{r}$ is a scalar estimate for the reward rate.
If both $q$ and $\tau$ are linear,
the above updates are the same as~\eqref{eq:gq1-update} with $\eta = 0$.

\section{Implementation Details}
\label{sec:impl}

\subsection{Boyan's Chain}
The state features we use are provided in Section C.1 of \citet{zhang2020gradientdice}.
% Diff-SGQ has a hyperparamter $\zeta$ to control the relative update rate of $w$ and $\hat r$.
% We can also introduce this $\zeta$ into Diff-GQ1 and Diff-GQ2 to control the relative update rate between the primal variable and the dual variable (see, e.g., 
% \citet{sutton2009fast}).
% As our empirical study focuses on the constant learning rate setting,
% we set $\zeta = 1$ for all algorithms.

\subsection{MuJoCo}
The dataset is composed by running the behavior policy for $10^6$ steps.
For GradientDICE,
we use neural networks to parameterize $\tau$ and $\nu$.
For Diff-SGQ,
we use neural networks to parameterize $q$.
For Diff-GQ1 and Diff-GQ2,
we use neural networks to parameterize $q$ and $\tau$.
All the networks have the standard architecture,
which are exactly the same as \citet{zhang2020gradientdice}.
They are two-hidden-layer networks with each hidden layer consisting of 64 hidden units and ReLU \citep{nair2010rectified} activation function.
The output layer does not have nonlinear activation function.
The $\hat{r}$ for all algorithms is always a global scalar parameter.
For GradientDICE, 
we find using an additional parameter $\hat{r}$ for reward rate prediction performs better and is much more computationally efficient than using $\frac{1}{N} \sum_{k=1}^{N} \tau(S_k, A_k) R_k$,
where $N = 10^6$ is the number of transitions in the dataset. 
As recommended by \citet{zhang2020gradientdice},
we use SGD to train all algorithms and do not use ridge regularization.
We sample $100$ transitions each step to form a minibatch for training.
Diff-GQ2 performs one gradient update every two steps.
For Diff-SGQ,
we update the target network every 100 steps.

\section{Other Experimental Results}
\subsection{Simulation of Assumption~\ref{assu:mupi}}
\label{sec:assumption-simulation}
We provide simulation results investigating when Assumption~\ref{assu:mupi} is likely to hold.
For each trial,
we first generate a random $P_\pi \in \R^{\nsa \times \nsa}$,
each row of which is randomly sampled from a simplex.
We then compute its stationary distribution $d_\pi$ analytically. 
The sampling distribution $d_\mu$ is composed by adding Gaussian noise to $d_\pi$,
i.e.,
$d_\mu(s, a) = d_\pi(s, a) + \mathcal{N}(0, \sigma^2)$.
We then normalize $d_\mu$ by $\frac{1}{\sum_{s, a}d_\mu(s, a)}$.
If the normalized $d_\mu$ still does not lie in a simplex,
we then apply softmax to $d_\mu$.
We use $D = diag(d_\mu)$
and generate the feature matrix $X \in \R^{\nsa \times K}$ randomly,
each element of which is sampled from $\mathcal{N}(0, 1)$ and $K$ is uniformly randomly sampled from $\qty{1, 2, \dots, \nsa}$.
We then analytically compute if $F$ in Assumption~\ref{assu:mupi} is positive semidefinite or not.
We conduct $10^4$ trials for each $(\nsa, \sigma, \xi)$ and report the probability that $F$ is positive semidefinite in Tables~\ref{tab:sim1} and~\ref{tab:sim2}.

\begin{table}[h]
\centering
\begin{tabular}{c|c|c|c|c|c}
 & $\sigma=0$ & $\sigma=0.001$ & $\sigma=0.01$ & $\sigma = 0.1$ & $\sigma=1$ \\ \hline
$\nsa=5$ & 0.70&0.69&0.70&0.65&0.52 \\\hline
$\nsa=10$ & 0.64&0.65&0.63&0.56&0.42 \\\hline
$\nsa=50$ & 0.55&0.50&0.44&0.41&0.36 \\\hline
$\nsa=100$ & 0.52&0.42&0.43&0.38&0.35 \\\hline
\end{tabular}
\caption{\label{tab:sim1} The probability of $F$ being positive semidefinite with $\xi = 0.9$.}
\end{table}
\begin{table}[h]
\centering
\begin{tabular}{c|c|c|c|c|c}
 & $\sigma=0$ & $\sigma=0.001$ & $\sigma=0.01$ & $\sigma = 0.1$ & $\sigma=1$ \\ \hline
$\nsa=5$ & 0.92&0.92&0.91&0.77&0.58 \\\hline
$\nsa=10$ & 0.92&0.92&0.84&0.68&0.50 \\\hline
$\nsa=50$ & 0.93&0.68&0.53&0.48&0.42 \\\hline
$\nsa=100$ & 0.93&0.51&0.49&0.45&0.42 \\\hline
\end{tabular}
\caption{\label{tab:sim2} The probability of $F$ being positive semidefinite with $\xi = 0.99$.}
\end{table}

\subsection{Additional Results on MuJoCo}
\label{sec:mujoco-expt}

The results on MuJoCo tasks with $\sigma = 0.1$ and $\sigma = 0.5$ are reported in Figure~\ref{fig:mujoco_result_2}.
\begin{figure*}[h]
\centering
\includegraphics[width=\textwidth]{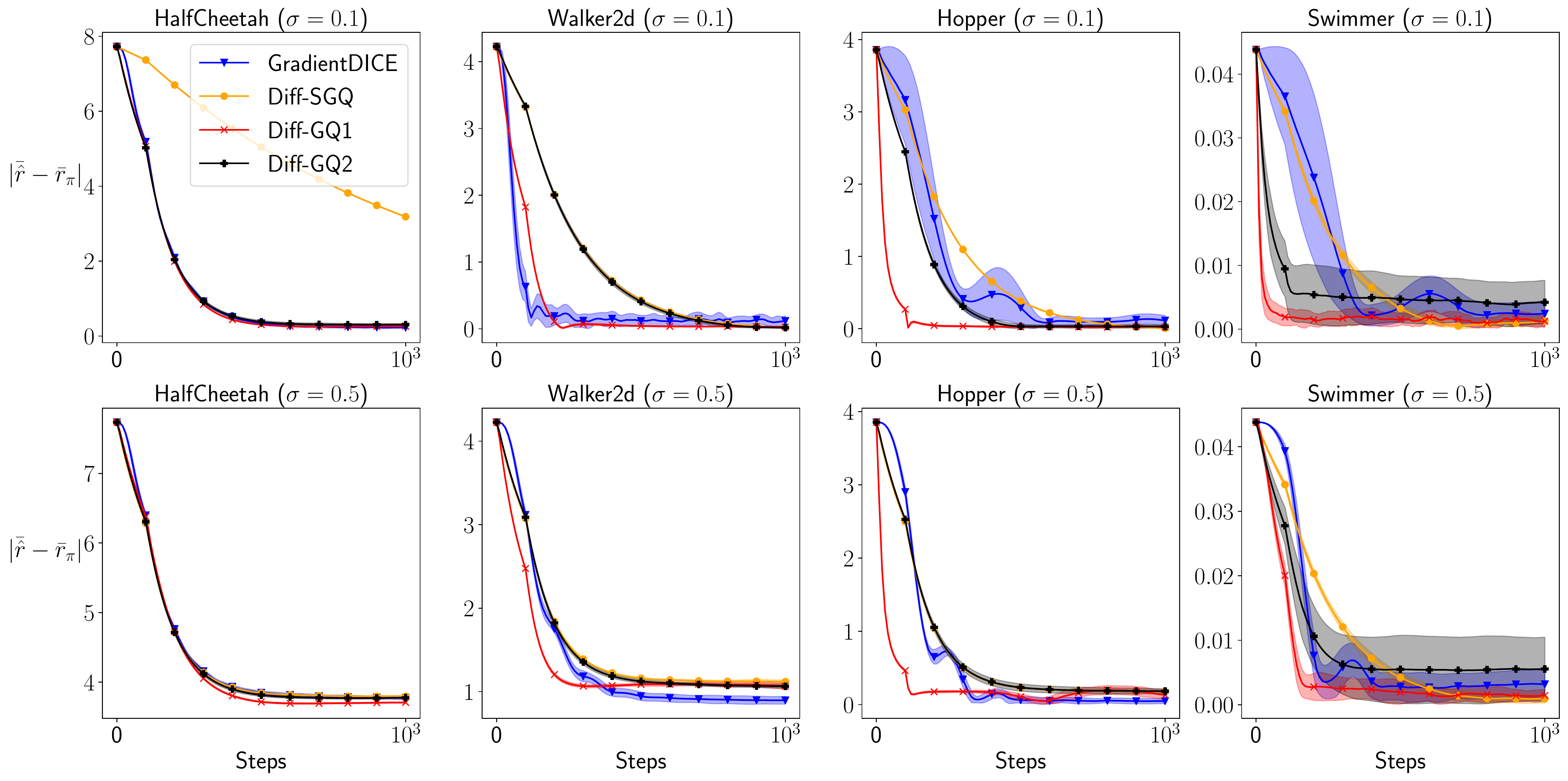}
\caption{\label{fig:mujoco_result_2}
MuJoCo tasks with with neural network function approximation.
$\bar{\hat{r}}$ is the average $\hat{r}$ of recent 100 steps.
}
\end{figure*}

\section{Other Differential Gradient $Q$ Evaluation Algorithms}
In this section,
we briefly discuss two other variants of Differential Gradient $Q$ Evaluation algorithms, Diff-GQ3 and Diff-GQ4.
Diff-GQ1 uses $Y$ as the feature matrix for both the primal variable $u$ and the dual variable $\nu$.
Consequently, to ensure the objective $J_{1, \eta}(u, \nu)$ is strictly concave in $\nu$,
$C$ has to be positive definite, i.e.,
Assumption~\ref{assu:nonconstant} is assumed.
Diff-GQ2 uses $X$ as the feature matrix for both the primal variable $w$ and the dual variable $\nu$.
Consequently, to obtain an estimate of the reward rate from $w$,
two i.i.d. samples are required.
To combine the advantages of both Diff-GQ1 and Diff-GQ2,
Diff-GQ3 uses $Y$ as the feature matrix for the primal variable $u$ but $X$ as the feature matrix for the dual variable $\nu$.
Diff-GQ3 considers the following MSPBE:
\begin{align}
\text{MSPBE}_3(u) = \norm{\Pi_X \bar \delta(u)}_D^2.
\end{align}
Similar to the derivation of Diff-GQ1,
we arrive at the Diff-GQ3 update:
\begin{align}
\label{eq diff gq3}
\delta_k &\doteq R_k - e_1^\top u_k + y_k'^\top u_k - y_k^\top u_k, \\
\nu_{k+1} &\doteq \nu_k + \alpha_k (\delta_k - x_k^\top \nu_k) x_k, \\
u_{k+1} &\doteq u_k + \alpha_k (y_k - y_k' + e_1)x_k^\top \nu_k - \alpha_k \eta u_k.
\end{align}
\begin{theorem}
\label{thm gq3 convergence}
If Assumptions~\ref{assu:chain},~\ref{assu:linearly independent}, \&~\ref{assu: stepsize} hold,
and either $\eta > 0$ or that $A_3$ is nonsingular holds, 
almost surely,
the iterate $\qty{u_k}$ generated by Diff-GQ3~\eqref{eq diff gq3}
converges to
$\tilde u^*_\eta \doteq -(\eta I + A_3^\top C_2^{-1} A_3)^{-1}A_3^\top C_2^{-1}b_2$, where
\begin{align}
A_3 \doteq X^\top D(P_\pi - I)Y - X^\top d_\mu e_1^\top.
\end{align}
% Further, if $A_3$ is invertible, then for $\eta = 0$, $\qty{u_k}$ converges almost surely to the $u_\text{TD}$ defined in \eqref{eq: TD fixed point form 1}.
\end{theorem}
The proof is the same as the proof of Theorem~\ref{thm:gq1-convergence} up to change of notations and thus omitted.
If $\eta = 0$ and $A_3$ is nonsingular,
it is easy to see that
\begin{align}
\tilde u_0^* = - A_3^{-1}b_2
\end{align}
and $\tilde u_0^*$ is the unique minimizer of MSPBE$_3(u)$.
Though in the tabular setting (i.e., $X = I$),
this $u_0^*$ is the TD fixed point $u_\text{TD}$,
in general they are not the same.
Moreover,
in Diff-GQ3,
we apply ridge regularization to $u = [\hat r, w^\top]^\top$.
If ridge is applied to only $w$ like Diff-GQ1 and Diff-GQ2,
the current proof of Theorem~\ref{thm gq3 convergence} will not hold.
We leave more analysis of Diff-GQ3 for future work.

We now briefly describe Diff-GQ4.
Given a reward rate estimate $\hat r$, 
we define 
\begin{align}
\text{MSPBE}_4(w; \hat r) \doteq \norm{\Pi_X(r - \hat r \tb{1} + P_\pi Xw - Xw)}_D^2.
\end{align}
Importantly, in MSPBE$_4$, 
$\hat r$ is fixed and is not a learnable parameter of this MSPBE.
By contrast, in MSPBE$_3$, both $\hat r$ and $w$ are learnable parameters of the MSPBE.
Diff-GQ4 updates $\hat r$ in the same way as Diff-SGQ but updates $w$ following $\nabla_w \text{MSPBE}_3(w; \hat r)$ under the current $\hat r$, i.e.,
\begin{align}
\hat r_{k+1} &\doteq \hat r_k + \alpha_k (R_k + x_k'^\top w_k - x_k^\top w_k - \hat r_k), \\
\nu_{k+1} &\doteq \nu_k + \alpha_k (R_k - \hat r_k + x_k'^\top w_k - x_k^\top w_k - x_k^\top \nu_k) x_k, \\
w_{k+1} &\doteq w_k + \alpha_k (x_k - x_k')x_k^\top \nu_k - \alpha_k \eta w_k.
\end{align}
Our preliminary work confirms the convergence of Diff-GQ4 when $\eta$ is sufficiently large.
We leave the analysis of Diff-GQ4 with a general $\eta$ and its fixed point for future work.

\end{document}